\documentclass{article} 
\usepackage[preprint]{colm2026_conference}
\usepackage{xcolor}   
\usepackage{microtype}
\usepackage{graphicx}
\usepackage{subfig}
\usepackage{booktabs} 
\usepackage[pagebackref]{hyperref}
\usepackage{url}
\usepackage{comment}
\usepackage{amsmath}
\usepackage{amssymb}
\usepackage{mathtools}
\usepackage{amsthm}
\usepackage{bm}
\usepackage{tcolorbox}
\usepackage{soul}
\usepackage{color}
\usepackage[textsize=tiny]{todonotes}
\hypersetup{
    colorlinks=true,
    citecolor =cyan,
    linkcolor=magenta,
    urlcolor=blue,
}
\usepackage{lineno}
\definecolor{myframe}{HTML}{E9EFFD}
\definecolor{myback}{HTML}{F0F5F6}

\usepackage{microtype}      
\usepackage{graphicx}
\usepackage{booktabs}       
\usepackage[ruled,algo2e]{algorithm2e}
\usepackage{algorithm}
\usepackage{comment}
\usepackage{nicefrac} 

\usepackage{amsmath}
\usepackage{amssymb}
\usepackage{mathtools}
\usepackage{amsthm}
\usepackage{comment}
\usepackage{cleveref}
\usepackage{minitoc}
\theoremstyle{plain}
\newtheorem{theorem}{Theorem}[section]
\newtheorem{proposition}[theorem]{Proposition}

\theoremstyle{definition}

\theoremstyle{remark}

\usepackage[textsize=tiny]{todonotes}
\usepackage[labelfont=it,font+=smaller]{caption}
\usepackage[labelfont=it]{subcaption}
\usepackage{makecell}
\usepackage{arydshln}
\usepackage{multirow}
\usepackage{soul}
\usepackage{color}

\definecolor{lightblue}{rgb}{0.7725490196, 0.89019607843, 0.9294117647}
\definecolor{darkblue}{rgb}{0.11764705882, 0.11764705882, 0.54509803921}
\usepackage{rotating}

\usepackage{xspace}
\usepackage{bbm}
\input{mathlig}
\usepackage{mathtools}
\usepackage{relsize}
\usepackage{mathrsfs}
\usepackage{dsfont}
\DeclarePairedDelimiterX{\inp}[2]{\langle}{\rangle}{#1, #2}
\makeatletter
\newcommand*\bigcdot{\mathpalette\bigcdot@{.5}}
\newcommand*\bigcdot@[2]{\mathbin{\vcenter{\hbox{\scalebox{#2}{$\m@th#1\bullet$}}}}}
\makeatother

\newcommand{\muspace}{\mspace{1mu}}

\DeclareRobustCommand{\scond}{\mathchoice{\muspace\vert\muspace}{\vert}{\vert}{\vert}}
\mathlig{|}{\scond}

\DeclareRobustCommand{\discint}{\mathchoice{\mspace{-1.5mu}:\mspace{-1.5mu}}{\mspace{-1.5mu}:\mspace{-1.5mu}}{:}{:}}
\mathlig{::}{\discint}
\newcommand{\suchthat}{\mathchoice{\colon}{\colon}{:\mspace{1mu}}{:}}

\newcommand{\Cc}{\mathcal{C}}

\newcommand{\Mc}{\mathcal{M}}

\newcommand{\Tc}{\mathcal{T}}

\newcommand{\Xc}{\mathcal{X}}

\newcommand{\ob}{{\mathbf o}}

\DeclareMathOperator\E{\mathsf{E}}

\def\textiid{i.i.d.\@\xspace}
\newcommand\iid{\ifmmode\text{ i.i.d. } \else \textiid \fi}

\def\mathllap{\mathpalette\mathllapinternal}
\def\mathllapinternal#1#2{%
  \llap{$\mathsurround=0pt#1{#2}$}}

\def\clap#1{\hbox to 0pt{\hss#1\hss}}
\def\mathclap{\mathpalette\mathclapinternal}
\def\mathclapinternal#1#2{%
  \clap{$\mathsurround=0pt#1{#2}$}}

\let\oldstackrel\stackrel
\renewcommand{\stackrel}[2]{\oldstackrel{\mathclap{#1}}{#2}}

\DeclarePairedDelimiterX{\infdivx}[2]{(}{)}{%
  #1\;\delimsize\|\;#2%
}

\renewcommand{\hbar}{h\mathllap{\overline{\vphantom{h}\hphantom{\rule{4.6pt}{0pt}}}\mspace{0.77mu}}}

\catcode`~=11 
\newcommand{\urltilde}{\kern -.06em\lower -.06em\hbox{~}\kern .02em}
\catcode`~=13 

\DeclarePairedDelimiterX{\norm}[1]{\lVert}{\rVert}{#1}
\DeclarePairedDelimiterX{\abs}[1]{\lvert}{\rvert}{#1}

\usepackage{xparse}

\let\oldpartial\partial
\renewcommand*{\partial}{\mathop{}\!\oldpartial}
\newcommand{\defeq}{\mathrel{\mathop{:}}=}

\colorlet{tablerowcolor}{lightblue}

\renewcommand{\E}{\mathbb{E}}

\usepackage{enumitem}

\usepackage{dashrule}

\usepackage{tcolorbox}\tcbuselibrary{skins}
\tcbuselibrary{breakable}

\tcbset{gray_box/.style={
    enhanced,
    breakable,
    size=fbox,
    boxrule=2pt,
    arc=2mm,
    auto outer arc,
    left=1pt,
    right=1pt,
    top=1pt,
    bottom=1pt,
}}

\tcbset{red_box/.style={
    enhanced,
    size=fbox,
    boxrule=2pt,
    arc=2mm,
    auto outer arc,
    left=1pt,
    right=1pt,
    top=1pt,
    bottom=1pt,
    colback=red!10,
    colframe=red!75!black,
    coltitle=white, 
}}

\tcbset{
    brown_box/.style={
        enhanced,
        colback=brown!10,
        colframe=brown!80!black,
        coltitle=white,
        arc=2mm,
        auto outer arc,
        left=1pt,
        right=1pt,
        top=1pt,
        bottom=1pt,
        boxrule=2pt,
    }
}

\definecolor{green_color}{RGB}{0, 150, 0}
\tcbset{green_box/.style={
    enhanced,
    breakable,
    size=fbox,
    boxrule=2pt,
    arc=2mm,
    auto outer arc,
    left=1pt,
    right=1pt,
    top=1pt,
    bottom=1pt,
    colback=green!8,
    colframe=green_color,
    coltitle=white, 
}}

\definecolor{blue_color}{RGB}{8, 104, 172}
\tcbset{blue_box/.style={
    enhanced,
    breakable,
    size=fbox,
    boxrule=2pt,
    arc=2mm,
    auto outer arc,
    left=1pt,
    right=1pt,
    top=1pt,
    bottom=1pt,
    colback=blue!5,
    colframe=blue_color,
    coltitle=white, 
}}

\RenewDocumentCommand\eqref{D<>{Eq.}om}{%
\IfNoValueTF{#2}
{#1~\oldeqref{#3}}
{(#2 #1~\textup{\ref{#3}})}%
}

\usepackage{scrwfile}
\TOCclone[\appendixname]{toc}{atoc}
\newcommand\StartAppendixEntries{}
\AfterTOCHead[toc]{%
  \renewcommand\StartAppendixEntries{\value{tocdepth}=-10000\relax}%
}
\AfterTOCHead[atoc]{%
  \edef\maintocdepth{\the\value{tocdepth}}%
  \value{tocdepth}=-10000\relax%
  \renewcommand\StartAppendixEntries{\value{tocdepth}=\maintocdepth\relax}%
}

\usepackage{booktabs}
\usepackage{siunitx}
\usepackage[strings]{underscore}

\title{Decocted Experience Improves Test-Time Inference\\ in LLM Agents}

\author{Maohao Shen\textsuperscript{1}\thanks{Co-first Authors.}\ ,\ 
Kaiwen Zha\textsuperscript{1}\footnotemark[1]\ ,\ 
Zexue He\textsuperscript{2}\footnotemark[1], \\
\bf{Zhang-Wei Hong\textsuperscript{3}\ ,\ 
Siru Ouyang\textsuperscript{4}\ ,\ 
J. Jon Ryu\textsuperscript{1}}, \\
\bf{Prasanna Sattigeri\textsuperscript{3}\ ,\ 
Suhas Diggavi\textsuperscript{5}\ ,\ 
Gregory Wornell\textsuperscript{1}} \\ \\
Massachusetts Institute of Technology\textsuperscript{1},
Stanford University\textsuperscript{2}, \\
MIT-IBM Watson AI Lab\textsuperscript{3},
University of Illinois at Urbana-Champaign\textsuperscript{4}, \\
University of California, Los Angeles\textsuperscript{5}\\
\textit{maohao@mit.edu, kzha@mit.edu, zexueh@stanford.edu}
}

\begin{document}

\ifcolmsubmission
\linenumbers
\fi

\maketitle

\begin{abstract}
There is growing interest in improving LLMs without updating model parameters. One well-established direction is test-time scaling, where increased inference-time computation (e.g., longer reasoning, sampling, or search) is used to improve performance. However, for complex reasoning and agentic tasks, naively scaling test-time compute can substantially increase cost and still lead to wasted budget on suboptimal exploration. In this paper, we explore \emph{context} as a complementary scaling axis for improving LLM performance, and systematically study how to construct better inputs that guide reasoning through \emph{experience}. We show that effective context construction critically depends on \emph{decocted experience}. We present a detailed analysis of experience-augmented agents, studying how to derive context from experience, how performance scales with accumulated experience, what characterizes good context, and which data structures best support context construction. We identify \emph{decocted experience} as a key mechanism for effective context construction: extracting essence from experience, organizing it coherently, and retrieving salient information to build effective context. We validate our findings across reasoning and agentic tasks, including math reasoning, web browsing, and software engineering.
\end{abstract}

\section{Introduction}
Recent progress in large language models (LLMs) has increasingly shifted attention from optimizing model parameters to improving inference-time behavior. A growing body of work shows that performance can often be improved through test-time scaling \citep{wu2024inference, snell2025scaling}. This shift is especially important for agentic systems, where an LLM must reason over long horizons, interact with tools or environments, and make a sequence of compounding decisions. In such settings, simply allocating more test-time compute is often not sufficient: extensive interaction can increase cost, and still spend substantial budget exploring sub-optimal paths~\citep{shen2025thinking, zeng2025satori}.

In this paper, we argue that \emph{context} constitutes another important axis for efficient scaling. Rather than increasing the output budget, one can instead provide the agent with carefully constructed inputs (i.e., context) to induce more efficient reasoning.
While this perspective is intuitively appealing, there has been limited systematic study on how to construct effective context, what characterizes a good context, and how performance scales with respect to it~\citep{mei2025survey, hua2025context}.
The most straightforward approach to construct context would be to handcraft prompts or demonstrations for a given task, but it is neither scalable nor reliable: handcrafted context must be redesigned for each new task, and it may fail to align with the actual strengths and failure modes of the deployed agent. 

A more natural source of high-quality context is \emph{experience}, e.g., through interaction with environments, an agent accumulates trajectories and feedback that contain reusable strategies, workflows, and pitfalls to avoid. Prior work has shown that such prior interaction traces can improve test-time performance~\citep{suzgun2025dynamic, ouyang2025reasoningbank, wei2025evomemory}. However, most existing methods focus on specific frameworks for leveraging experience, while a systematic understanding of how raw experience should be transformed into effective context remains limited. It remains unclear \textbf{(1) how performance scales with accumulated experience}, \textbf{(2) what characterizes effective context theoretically and empirically}, and \textbf{(3) what data structures best support context construction}. We address this gap by identifying \emph{decocted experience} as a key ingredient for successful experience-augmented agents: extracting the essence from experience, organizing it coherently, and retrieving salient information to construct effective context. Our main findings are as follows:
\begin{enumerate}[leftmargin=*]
\vspace{-0.5em}
\item In Section~\ref{sec:scaling_behavior}, we study what makes experience-augmented agents scalable along two coupled axes: input context size and accumulated experience size. We show that experience decoction through lesson distillation outperforms raw experience in agentic tasks (Figure~\ref{fig:experience_lesson}) and yields better input scalability (Figure~\ref{fig:context_scaling}). We further show that maintaining accumulated experience as memory is critical for scalability: performance exhibits a sweet spot at intermediate memory sizes that can outperform full memory (Figure~\ref{fig:memory_consolidation}), identifying memory consolidation as a second key factor alongside experience decoction.

\item In Section~\ref{sec:analysis}, we investigate how to characterize effective context. Theoretically, we first show that context with higher information gain about the agent's generated output leads to more efficient inference (Proposition~\ref{prop:efficiency} \& Figure~\ref{fig:empirical_validation}). Empirically, we show that effective context should balance the trade-off between relevance and diversity, and that this trade-off is associated with the agent's performance improvement (Figure~\ref{fig:context_quality}).

\item In Section~\ref{sec:concept_tree}, we further study a memory structure tailored for experience-augmented agents. We show that extracting concepts from experience and organizing them into a tree structure across different concept groups improves context construction by encouraging retrieval from more diverse yet still relevant concept groups, leading to better performance than baseline methods (Figure~\ref{fig:concept_tree_webshop}).
\end{enumerate}
\section{Problem Setup}
\label{sec:formulation}
We study how an agent should \emph{craft context from experience} to efficiently solve new tasks. 

\begin{figure}[!t]
\centering
\includegraphics[width=0.93\textwidth]{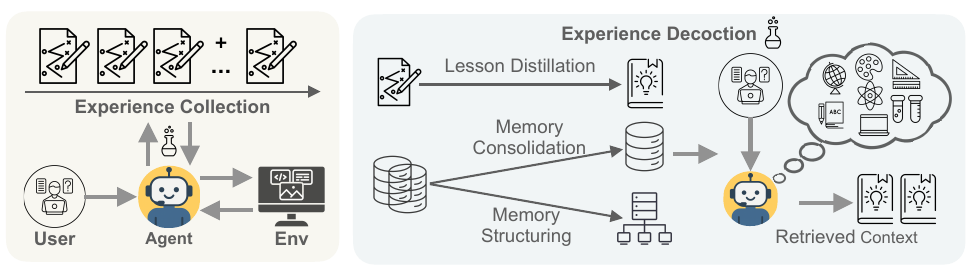}
\caption{\textbf{Experience-Augmented Agent.} The agent accumulates experience from past interactions, decocts it into effective context for improved inference at test time, i.e., distilling lessons from experience, organizing the experience memory, and finally retrieving salient information from it.}
\label{fig:teaser}
\vspace{-1.5em}
\end{figure}

\textbf{Agent.}
Throughout, we assume that an agent is built upon a frozen LLM $\pi_\theta : \mathcal{C} \to \mathcal{Y}$, where $\theta$ denotes pretrained parameters and $\mathcal{C}$ is the space of textual contexts. 
Given a constructed context $c \in \mathcal{C}$, the model generates $\hat{y} = \pi_\theta(c)$, where $\hat{y}$ can either be a single-turn chain-of-thought (CoT) trace or a trajectory 
$\hat{y} = (s_1, a_1, \dots, s_T, a_T)$, where $s_t$ denotes an intermediate state (e.g., environment state) and $a_t$ denotes an action. 

\textbf{Training: Memory Construction from Experience.}
The goal of this paper is to study how to construct an informative yet succinct context based on \emph{experience}, which is formally defined as follows, and
we assume that the agent has access to a set of $N$ problems $\{x_i\}_{i=1}^N$, where, for each $i$, $m$ pairs of trajectory and final reward $\ob_i\defeq\{(y_i^{(j)},r_i^{(j)})\}_{j=1}^m$ are given collectively as \emph{experience}. The experience may not be necessarily from the same agent, but all the experiments in this manuscript assume self-experience. 
We assume the same $m$ for different problems for simplicity.
Given a formatting function $\Psi$, we define $z_i \defeq \Psi(x_i, \ob_i)$ for each $i\in[N]$, and call ${\Mc}=\{z_i\}_{i=1}^N$ the \emph{raw memory}. Since the size of the raw memory grows linearly in $N$, the number of previous examples, and thus it may become inefficient to maintain the entire raw memory at test time.
Instead, we consider a memory organization mechanism $\mu$ that \emph{organizes} ${\Mc}$ into a more succinct memory $\tilde{\Mc}\defeq \mu({\Mc})$.

\textbf{Test-time: Inference with Memory.}
Having built the organized memory $\tilde{\Mc}$,
given a new query $x \in \Xc$, we consider a retrieval mechanism 
$R : \mathcal{X} \times \tilde{\Mc} \to 2^{\tilde{\Mc}}$ which selects a subset of $K$ memory entries relevant to $x$: $R(x; \tilde{\Mc}) = \{ z_{i_1}, \dots, z_{i_K} \}$.
Here, $K\ge 1$ is a hyperparameter.
A context constructor $C : \mathcal{X} \times 2^{\tilde{\Mc}} \to \mathcal{C}$ then produces a context $c = C(x, R(x;\tilde{\Mc}))$, which is provided to the agent as a prompt for inference, i.e., $\hat{y} = \pi_\theta(c)$. Throughout the paper, we assume a simple concatenation-based construction.

\textbf{Scope of Analysis.}
Given an agent $\pi_\theta$ with past experience $\{(x_i, \ob_i)\}_{i=1}^N$, several quantities in the system may grow large in practice, including the number of trajectories per problem (i.e., $m$), the amount of accumulated experience (i.e., $N$), and the size of the retrieved context (i.e., $K$). Our goal is to understand what makes such a system scalable through \emph{experience decoction}: how raw experience should be distilled, organized, and retrieved to construct the context so as to maximize the agent’s expected performance on unseen test problems.

\subsection{A Minimal Instantiation of the Experience-Augmented Agent} \label{sec:instantiation}
To understand the importance of experience decoction, we begin with a minimal instantiation of the experience-augmented agents that do not rely on sophisticated design choices: \textbf{(1) Memory Organization}: We begin with a flat memory and retain only one successful trajectory out of the $m$ trajectories for each entry in order to reduce the retrieved context length, i.e., $\tilde{\Mc} = \{(x_i, y_i)\}_{i=1}^N$. \textbf{(2) Retrieval}: Given a query problem $x$ at test time, we compute embeddings using a fixed encoder $\pi_e : \mathcal{X} \to \mathbb{R}^d$ and define a similarity function
$s(x, x_i) = \langle \pi_e(x), \pi_e(x_i) \rangle$, where $\langle\cdot,\cdot\rangle$ denotes the standard inner product. 
The retrieval mechanism selects the top-$K$ entries according to similarity:
$R(x, \tilde\Mc)=\operatorname{TopK}_{z_i \in \tilde\Mc}
\; s(x, x_i).$ Our analyses in Section~\ref{sec:scaling_behavior} and Section~\ref{sec:analysis} are based on this flat experience memory and basic retriever, while in Section~\ref{sec:concept_tree} we further investigate a more advanced organization system.

\subsection{Experimental Setup} \label{subsec:setup}
\textbf{Tasks and Datasets.}
We study three representative tasks: \textbf{Mathematical Reasoning (Math Reasoning)}, \textbf{Web Browsing (WebShop)}, and \textbf{Software Engineering (SWE)}. For math reasoning, we build memory from DAPO-Math~\citep{yu2025dapo} and evaluate on AMC, AIME, HMMT, and BeyondAIME using exact-match correctness. For WebShop, we use WebShop~\citep{yao2022webshop}, where an agent searches an online shopping environment to satisfy user requests, and performance is measured by the quality of the final purchase. For SWE, we use SWE-bench~\citep{jimenez2023swe}, where an agent resolves real GitHub issues by editing code in a repository environment, and success is determined by whether the generated patch passes the evaluation tests. More details are included in Appendix~\ref{app:datasets}.

\textbf{Models.}
We use Seed-OSS-36B-Instruct~\citep{seed2025seed-oss} as the base language model $\pi_\theta$ in our main experiments. For reasoning tasks, the model generates a single-turn chain-of-thought (CoT) trace followed by a final answer. For agentic tasks, it is equipped with the required external tools and operates in an iterative action--observation loop until termination. We also extend our analysis to another open-source model, GPT-OSS-20B~\citep{openai2025gptoss120bgptoss20bmodel}. We use Qwen3-Embedding-4B~\citep{qwen3embedding} as the embedding model $\pi_e$.

\textbf{Evaluation.}
We evaluate how leveraging context improves both the effectiveness and efficiency of the LLM agent under a fixed test-time budget. We use two metrics: (1) \textbf{Effectiveness:} For a query $x$, we measure the average quality of multiple attempts as $\mathrm{avg}\_m = \frac{1}{m} \sum_{j=1}^{m} r^{(j)}$. (2) \textbf{Efficiency:} For reasoning tasks (e.g., Math Reasoning), we measure efficiency by the number of generated output tokens (i.e., the CoT length). For agentic tasks (e.g., WebShop and SWE), we use the number of interaction steps $T$.

\section{What Makes Experience-Augmented Agent Scalable?} 
\label{sec:scaling_behavior}
As discussed in Section~\ref{sec:formulation}, an experience-augmented agent faces scaling challenges along two axes: the input context can become large since retrieved experiences can be lengthy, and the memory itself can grow substantially as experience accumulates. In this section, we argue that \emph{experience decoction} is the key to scalability: locally, it distills raw experience into concise lessons, and globally, it consolidates the memory by removing redundancy.

\subsection{Input Size Scaling: How to Construct Effective Context?}
\label{subsec:context}
\textbf{Distilled Lessons Are More Effective Than Raw Experience in Agentic Tasks.}
Given a raw experience memory, i.e., $\tilde{\Mc}=\Mc=\{(x_i,\ob_i)\}_{i=1}^N$, where $\ob_i\defeq \{(y_i^{(j)},r_i^{(j)})\}_{j=1}^{m}$, directly using raw experience as context can be sub-optimal due to lengthy trajectories and large $m$. This motivates a compression mechanism that we refer to as \emph{lesson distillation}. Specifically, past trajectories are distilled into a single reusable lesson that summarizes transferable reasoning patterns: $\ell_i = \pi_{\theta}(x_i, \ob_i)$, and the distillation is performed by the agent itself. The resulting context becomes $c = \text{Concat}\big(x, \{(x_{i}, \ell_{i})\}_{i=1}^K\big)$.  To evaluate this idea, for a fixed $K$, we compare a system built on distilled lessons, i.e., $\{(x_i,\ell_i)\}_{i=1}^N$, against the raw experience-based minimal instantiation in Section~\ref{sec:instantiation}. Results are shown in Figure~\ref{fig:experience_lesson}. Compared to a vanilla agent without any context, both methods substantially improve effectiveness ($\mathrm{avg}\_m$) and efficiency. On reasoning tasks, lesson-based context slightly underperforms raw experience-based context, suggesting that detailed reasoning traces already provide sufficient guidance. In contrast, on WebShop and SWE tasks, lesson-based context consistently outperforms raw experience. Since Agentic trajectories contain extensive environment observations and interaction noise, distillation effectively filters irrelevant information and retains transferable decision-making tricks.

\begin{figure}[h]
    \centering
    \includegraphics[width=1.0\textwidth]{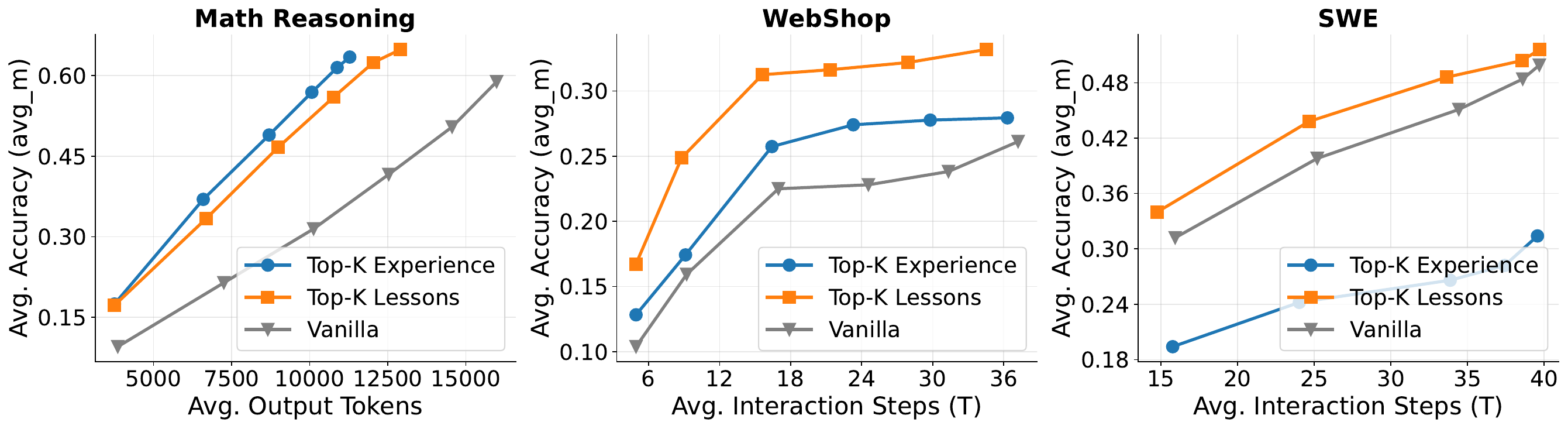}
\caption{\textbf{Raw Experience vs. Distilled Lesson as Context.} Both context construction approaches significantly outperform the vanilla agent without context. Raw experience is slightly stronger for mathematical reasoning, while distilled lessons yield better performance in agentic tasks (WebShop \& SWE), where trajectory-level observations are noisier and distillation helps.
}
\label{fig:experience_lesson}
\vspace{-0.5em}
\end{figure}

\textbf{Lesson Distillation Enables More Favorable Context Scaling.}
We further analyze the \emph{input size scaling behavior} of the two approaches by increasing $K$ in Top-$K$ retrieval. Figure~\ref{fig:context_scaling} reports performance as a function of input token length for both experience-based and lesson-based context construction. First, the lesson-based approach consistently achieves strong performance with a smaller input context size. For a fixed performance level, fewer input tokens are required compared to raw experience. This advantage is particularly pronounced in WebShop, where distilled lessons outperform raw trajectories. Second, as $K$ increases, we observe that the performance of the experience-based approach can degrade, especially when the input context becomes excessively long. This suggests that overly detailed and noisy demonstrations could be misleading. 
\begin{figure}[h]
    \centering
    \includegraphics[width=1.0\textwidth]{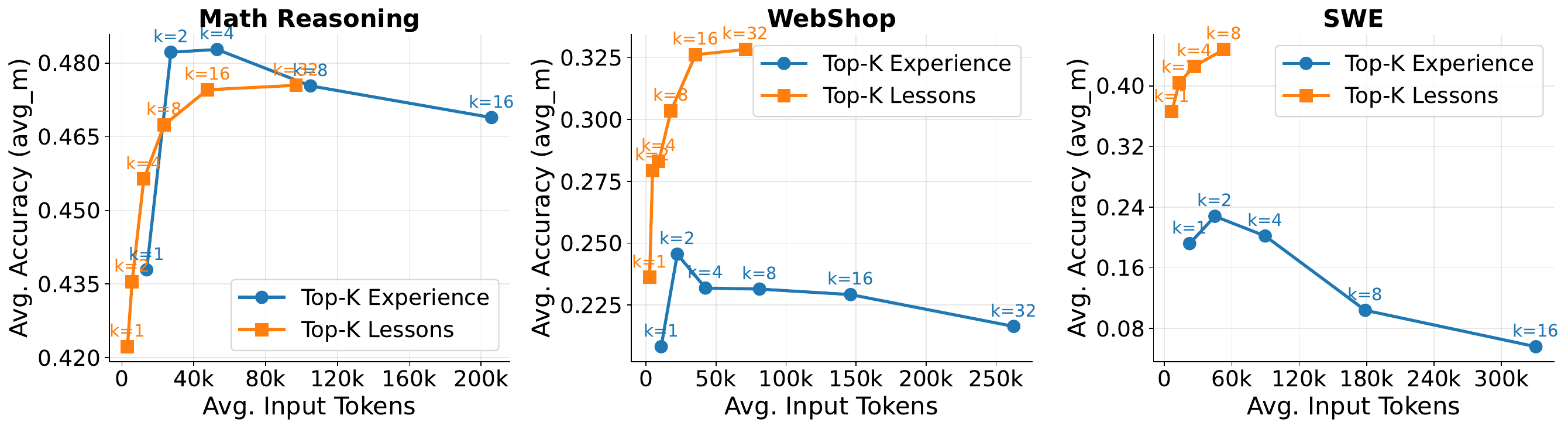}
\caption{\textbf{ Scaling Behavior.} Agent's performance as a function of input context length when increasing $K$ in Top-$K$ retrieval. Distilled lessons achieve better performance with fewer input tokens and remain more robust as context grows. In contrast, raw experience can degrade when the prompt becomes excessively long and noisy. Overall, lesson distillation acts as an effective context compression mechanism that extracts the essential information from experience. 
}
\label{fig:context_scaling}
\vspace{-1.5em}
\end{figure}

\subsection{Experience Size Scaling: How to Construct Effective Memory?}
\label{subsec:memory}
As the experience grows, memory size can also become arbitrarily large. A large memory can introduce redundancy: many entries may be similar and contain overlapping experience. Therefore, this motivates a memory consolidation mechanism, i.e., $\tilde{\Mc}\defeq \mu({\Mc})$.

\textbf{Clustering-Based Memory Consolidation.}
We adopt a clustering-based consolidation strategy. Let each memory entry be associated with an embedding vector $e_i \in \mathbb{R}^d$, derived from the problem statement and the distilled lesson, 
$e_i = \text{Normalize}\big((1-\alpha)\,\pi_e(x_i) + \alpha\,\pi_e(\ell_i)\big),$
where $\alpha \in [0,1]$ controls the contribution of lesson-level semantics. We partition the memory into $N^{\prime}$ clusters using standard $k$-means. For each cluster, we select a single representative entry that is closest to the cluster centroid. The consolidated memory becomes $\tilde{\Mc} = \{ (\tilde{x}_i,\tilde{\ell}_i) \}_{i=1}^{\tilde{N}}$, where $\tilde{N} < N$. We evaluate performance using $\tilde{\Mc}$ across different sizes.

For this study, we consider a new metric \emph{Avg. Retrieval Relevance}, i.e., the relevance score measured by embedding similarity between context retrieved from $\tilde{\Mc}$ and query sample. From Figure~\ref{fig:memory_consolidation} (a), we observe that \emph{retrieval relevance scores increase monotonically} with memory size, as expected, as memory consolidation is essentially a lossy compression.

\textbf{Memory Consolidation Exhibits a Sweet Spot.}
Next, we further evaluate agent's performance w.r.t. the memory size in Figure~\ref{fig:memory_consolidation} (b), our results reveal several key findings. First, consolidation exhibits a \emph{non-monotonic trend}: when $\tilde{N}$ is too small, aggressive compression removes essential information, causing noticeable performance degradation. When $\tilde{N}$ approaches the full memory size, consolidation provides little benefit. Second, there exists a \emph{sweet spot} at intermediate memory sizes where k-means consolidation achieves improved performance and efficiency compared to full memory. These results highlight a key insight: \emph{higher retrieval relevance does not necessarily yield better task performance}. This reveals a fundamental \emph{relevance--diversity tradeoff}: while semantic similarity helps retrieve related lessons, excessive similarity can reduce the diversity of problem-solving strategies and limit generalization. Overall, moderate memory consolidation improves retrieved context by reducing redundancy while preserving diversity, which is essential for experience-augmented agents.

\begin{figure}[h]
    \centering
    \begin{tabular}{rl}
    \subfloat[Retrieval Relevance vs. Consolidated Memory Size]{%
        \includegraphics[width=0.95\textwidth]{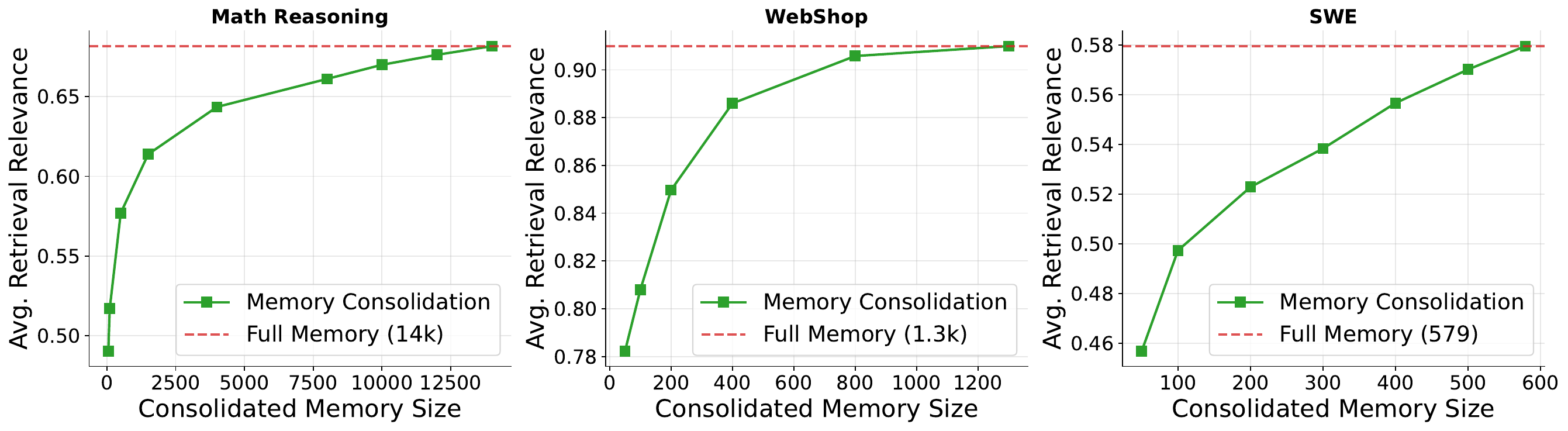}
    } \\
    \subfloat[Performance vs. Consolidated Memory Size]{%
        \includegraphics[width=0.95\textwidth]{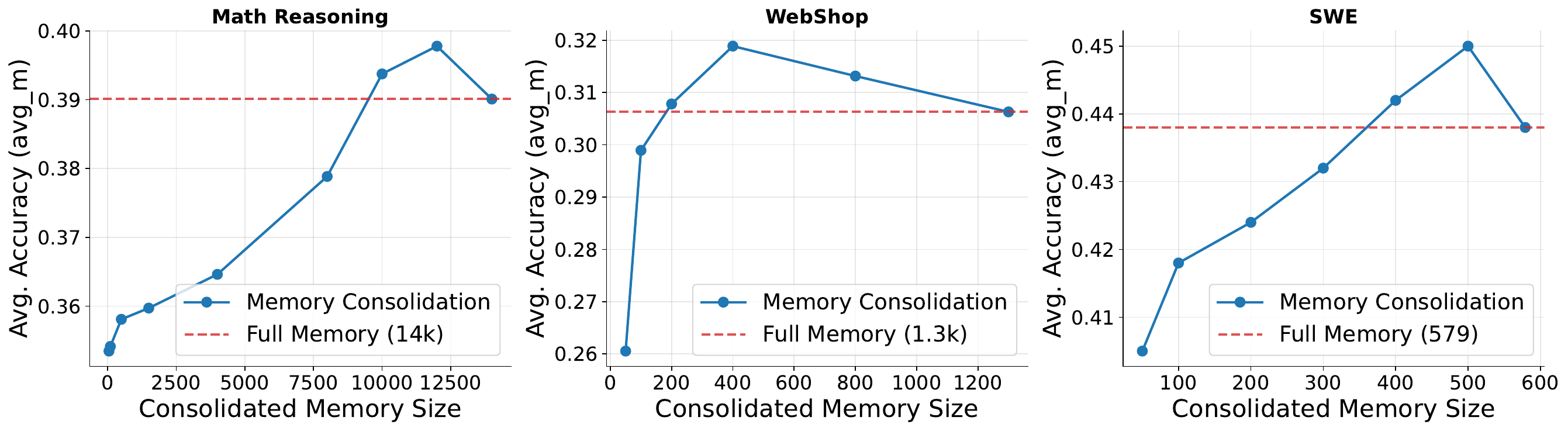}
    }
    \end{tabular}
\caption{\textbf{Experience Scaling Behavior via Memory Consolidation.} We evaluate experience memory consolidation across varying memory sizes. The red dashed line indicates full memory baseline performance. Memory consolidation achieves a sweet spot at intermediate sizes.}
\label{fig:memory_consolidation}
\vspace{-1em}
\end{figure}
\section{What Characterizes an Effective Context?}
\label{sec:analysis}
The previous section demonstrates that constructing context through decocted experience can effectively improve agent's test-time inference.
An important question in practice is how we can construct \emph{effective} context for a new task.
To answer this question, in this section, we first provide an intuitive theoretical insight on the characterization of good context, and then propose an empirical measure that forecasts the agent's performance improvement.

\textbf{Theoretical Analysis.}
In Figure~\ref{fig:experience_lesson}, we demonstrate that more informative context (constructed from distilled lesson) may yield better performance with fewer output tokens (or interaction steps). One may naturally ask whether there is a direct theoretical connection between the informativeness of context and effectiveness of inference. Here, we provide a quantitative yet intuitive relationship between them via information measures.

Fix a query $x \in \mathcal{X}$. Let $C$ denote the random variable corresponding to the context constructed for $x$, and let $Y$ denote the random trajectory generated by the agent given $(x,c)$. 
A natural quantity to capture the informativeness of a context $c$ is the \emph{information gain} of context $c$ for query $x$ defined as
\[
I(Y; C=c \mid X=x)
\defeq
H(Y \mid X=x) - H(Y \mid X=x, C=c).
\]
We note that the average of the information gain over the randomness in the context corresponds to the \emph{per-query mutual information} $I(Y;C \mid X=x)
\defeq \sum_{c\in\Cc}p(c|x)I(Y;C=c|X=x)$,
which is always non-negative. 
However, the information gain for each context can take either positive or negative values:
if $I(Y; C=c \mid X=x)\le 0$, context fails to reduce uncertainty, indicating that the context may be misleading or noisy; $I(Y; C=c \mid X=x)>0$ indicates that context is \emph{effective} in that it reduces output uncertainty.

To connect the information gain with inference efficiency, let $\tau$ denote the stopping time such that the first EOS token is generated at time $\tau+1$, such that $Y_{1:\tau}$ corresponds to the sequence of the non-EOS tokens.

\begin{proposition}[Context Efficiency Bound] 
\label{prop:efficiency}
Fix $X=x$ and $C=c$.
Suppose there exists a constant $h>0$ such that
\[
H\left(
Y_{1:\tau} \mid \tau=t, X=x, C=c
\right)
\ge th,
\quad \forall t\ge 0.
\]
In words, each token carries at least $h$ bits of uncertainty on average.
Then,
\[
\mathbb{E}[\tau \mid X=x, C=c]
\le
\frac{H(Y\mid X=x, C=c)}{h}
=\frac{H(Y\mid X=x)-I(Y ;C=c \mid X=x)}{h}.
\]
\end{proposition}

Given base uncertainty $H(Y \mid X=x)$ and the average entropy rate lower bound $h$, larger information gain $I(Y; C=c \mid X=x)$ leads to a strictly smaller upper bound on the expected trajectory length $\mathbb{E}[\tau \mid X=x,C=c]$. 
Intuitively, informative context would reduce uncertainty in the output
trajectory, effectively contracting the search space that the model must
explore during generation. As a result, the agent can reach solutions using fewer reasoning tokens or interaction steps.

\begin{figure}[t]
\centering
\begin{tabular}{rl}
    \subfloat[Conditional Entropy vs.\ Expected Output Length.]{%
        \includegraphics[width=0.44\textwidth]{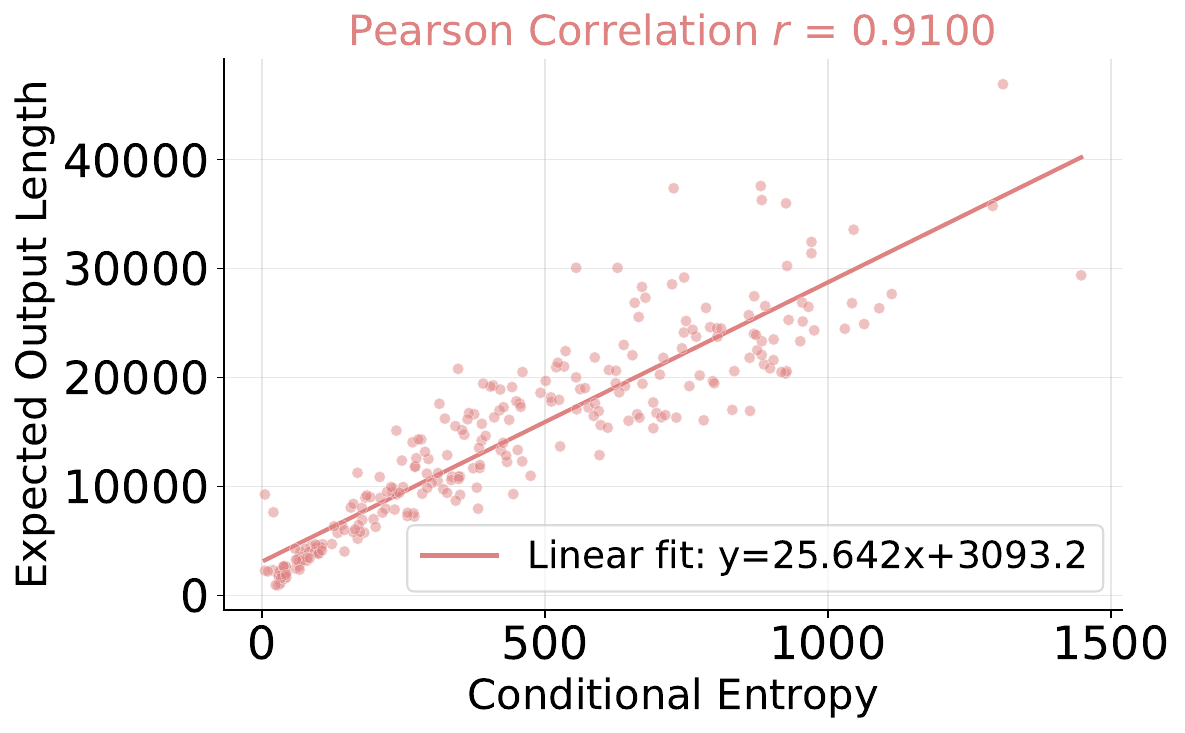}
    } &
    \subfloat[Distribution of Information Gain.]{%
        \includegraphics[width=0.44\textwidth]{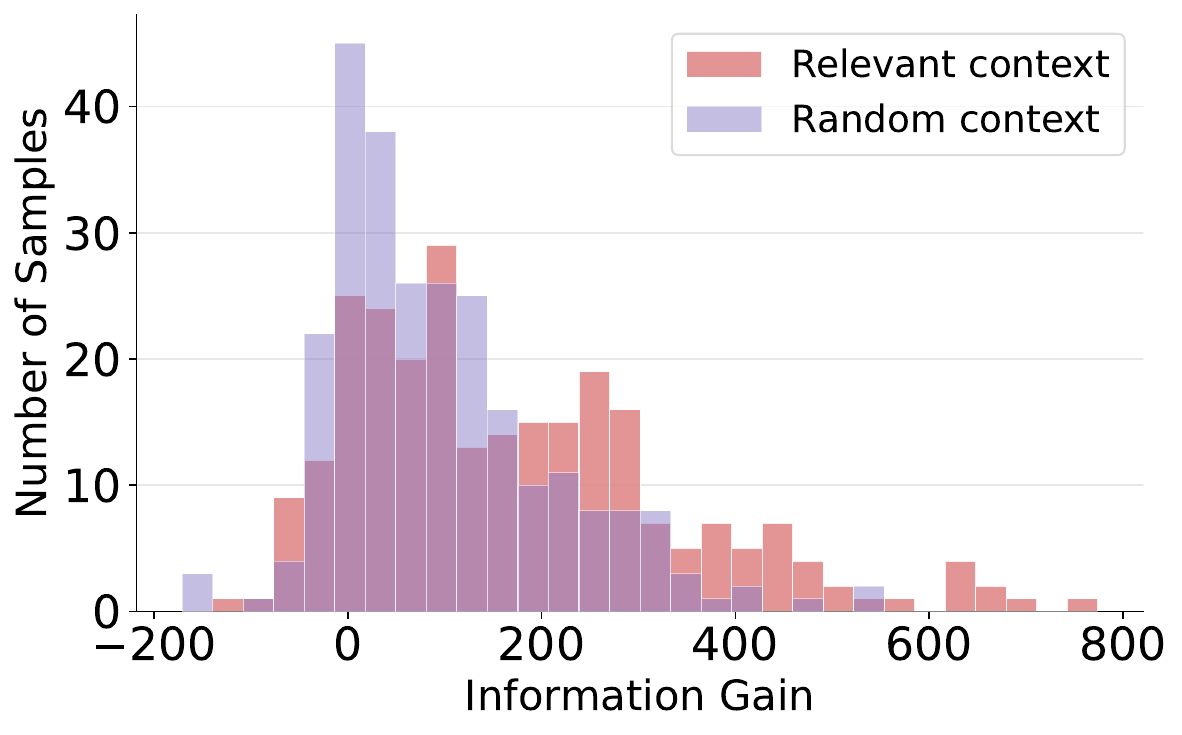}
    }
\end{tabular}
\caption{
\textbf{Empirical validation of Proposition~\ref{prop:efficiency}.}
Figure \textbf{(a)}'s strong linear correlation confirms that $\hat{H}(Y \mid X=x, C=c)$ tightly predicts the expected output length up to a constant.
Figure \textbf{(b)} shows that relevant context (retrieved lessons) yields higher information gain than random context.}
\label{fig:empirical_validation}
\vspace{-1.5em}
\end{figure}

\textbf{Empirical Validation.}
We empirically validate Proposition~\ref{prop:efficiency} by estimating the
conditional entropy $\hat{H}(Y \mid x,c)$ for each $(x,c)$.
Specifically, we compute a Monte Carlo estimator of the conditional entropy
using the token-level log-probabilities produced by the model. We first examine the predicted relationship between conditional entropy
and output length. Figure~\ref{fig:empirical_validation}(a) plots
$\hat{H}(Y \mid X=x,C=c)$ against the average output length across all evaluation samples from the math reasoning benchmarks. The two quantities exhibit a strong linear
correlation ($r = 0.91$), confirming that the
bound in Proposition~\ref{prop:efficiency} is tight up to a constant
factor: lower conditional entropy directly translates into shorter
generated trajectories.

Next, we investigate whether the \emph{information gain} captures
the \emph{quality} of the retrieved context. For each problem $x$, we estimate the empirical information gain $\hat{I}(Y;C=c\mid X=x)=\hat{H}(Y\mid X=x) - \hat{H}(Y\mid X=x,C=c)$, where $\hat{H}(Y\mid X=x)$ corresponds to the entropy under the no-context setting. Figure~\ref{fig:empirical_validation}(b) compares the distribution of $\hat{I}(Y;C=c\mid X=x)$ under two conditions: \emph{relevant context} (retrieved lessons) and \emph{random context}
(randomly sampled lessons). Relevant context produces substantially higher information gain (mean $\hat{I}=172.2$ vs.\ $97.2$ for random
context) and dominates on 76\% of evaluation samples. 

Together, these results provide an explanation for the context-scaling behavior observed in Section~\ref{sec:scaling_behavior}. Context construction methods that preserve more information in context
about the correct trajectory increase
$I(Y;C=c\mid X=x)$ and enable more efficient inference.

\begin{figure}[t]
\centering
\begin{tabular}{rl}
    \subfloat[Correlation Between Context Quality \& Performance Improvement]{%
        \includegraphics[width=0.44\textwidth]{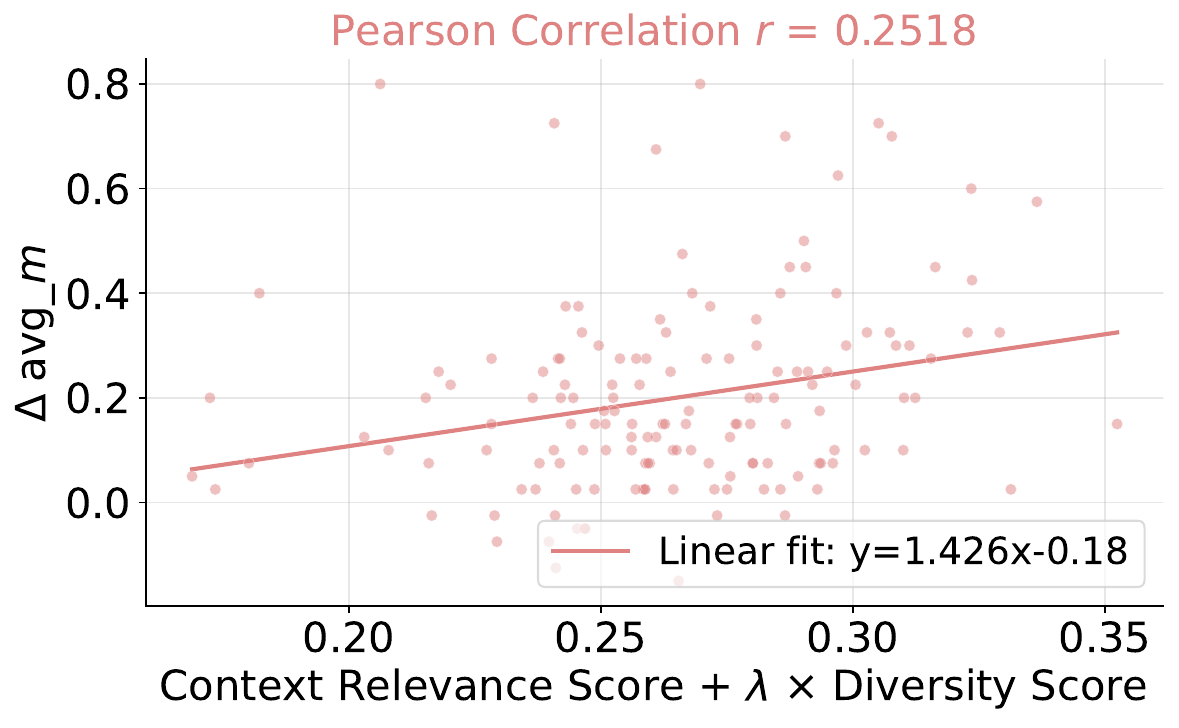}
    } &
    \subfloat[Trade-off Between Relevance \& Diversity]{%
        \includegraphics[width=0.44\textwidth]{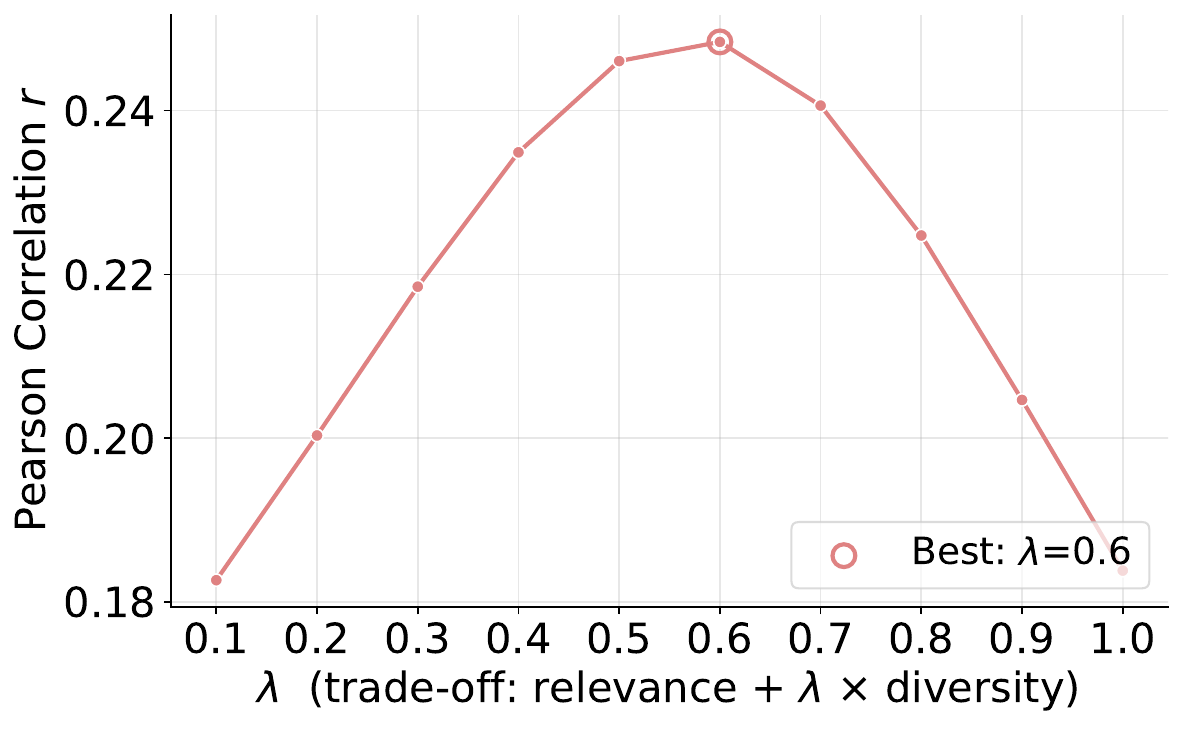}
    }
\end{tabular}
\caption{
\textbf{Correlation between Context Quality and Performance Improvement.}
(a) The relationship between the context quality score and the improvement measured in $\Delta \mathrm{avg}\_m(x)$. The positive Pearson correlation indicates that higher-quality contexts tend to yield larger performance gains. (b) Pearson correlation as the coefficient $\lambda$ varies. The correlation peaks around $\lambda = 0.6$, suggesting that balancing relevance and diversity leads to the most informative contexts.
}
\label{fig:context_quality}
\vspace{-1em}
\end{figure}
\textbf{What Makes a Good Context in Practice?}
The analysis above suggests that informative context can effectively reduce the length of the trajectory. In practice, however, the measure of information gain is not available before generation, so we need a proxy metric for context quality. 
Here, we hypothesize that effective context should balance two practically \emph{measurable} properties, \emph{relevance} and \emph{diversity}, defined as follows.
Recall that
$R(x,\tilde\Mc)=\{z_{i_1},\dots,z_{i_K}\}$ denotes the set of $K$ retrieved memory entries for a query $x$. For each $z \in R(x,\tilde\Mc)$, let $\pi_e(z) \in \mathbb{R}^d$ denote its embedding representation. \textbf{(1) Relevance.} We define the relevance score as the average similarity between the query embedding and the embeddings of the retrieved entries:
$\mathrm{Rel}_{R,\tilde\Mc}(x)
=
\frac{1}{K}
\sum_{z \in R(x,\tilde\Mc)}
\langle \pi_e(x), \pi_e(z) \rangle$.
A higher score indicates that the retrieved entries are more semantically aligned with the query.
\textbf{(2) Diversity.}
To measure redundancy within the retrieved set, we compute the negative average pairwise similarity among retrieved entries:
$\mathrm{Div}_{R,\tilde\Mc}(x)
=
-
\frac{1}{K(K-1)}
\sum_{\substack{z\neq z' \in R(x,\tilde\Mc)}}
\langle \pi_e(z), \pi_e(z') \rangle$.
A higher score indicates that the retrieved set is less redundant and covers a broader range of concepts or strategies. We define a retrieval quality score by combining these two terms
\[
Q_{R,\tilde{\Mc}}(x)
=
\mathrm{Rel}_{R,\tilde{\Mc}}(x)
+
\lambda \, \mathrm{Div}_{R,\tilde{\Mc}}(x),
\]
where $\lambda$ controls the relevance--diversity trade-off.
For each query $x$, we also quantify the performance improvement induced by the context:
\[
\Delta \mathrm{avg}\_m(x)
=
\mathrm{avg}\_m_{\mathrm{context}}(x)
-
\mathrm{avg}\_m_{\mathrm{vanilla}}(x),
\]
where $\mathrm{avg}\_m_{\mathrm{context}}(x)$ denotes the performance achieved using the context, and $\mathrm{avg}\_m_{\mathrm{vanilla}}(x)$ denotes the performance of the same agent with no context. 

To validate our hypothesis, we compute the Pearson correlation between the retrieval quality score $Q_{R,\tilde{\Mc}}(x)$ and $\Delta \mathrm{avg}\_m(x)$ across all evaluation samples. Figure~\ref{fig:context_quality}(a) demonstrates a positive Pearson correlation ($r = 0.25$), indicating that retrieved sets that are both relevant to the query and internally diverse tend to yield larger performance gains. Figure~\ref{fig:context_quality}(b) further examines the trade-off by varying $\lambda$. The correlation peaks around $\lambda=0.6$, suggesting that the most useful retrieved contexts should balance this trade-off.
\section{A Tree-Structured Memory for Experience-Augmented Agents} 
\label{sec:concept_tree}
The analyses in Section~\ref{sec:scaling_behavior} and Section~\ref{sec:analysis} identify key properties of effective memory for context construction, and in particular reveal the limitations of a \emph{flat} memory without structure. First, although memory consolidation (i.e., reducing the number of entries) can improve performance when applied at an appropriate level (Figure~\ref{fig:memory_consolidation}(b)), the optimal operating point is difficult to determine in advance. Second, Figure~\ref{fig:context_quality} shows that useful context should balance relevance and diversity, whereas standard similarity-based retrieval over a flat memory may fail to capture deeper conceptual relationships. Together, these observations suggest that effective memory should (i) adaptively control retrieval granularity without requiring explicit consolidation and (ii) promote diversity beyond surface-level similarity. Motivated by this, we propose a structured memory called a \emph{hierarchical concept tree}, tailored for experience-augmented agents. At a high level, the tree organizes memory entries via embedding-based clustering in a coarse-to-fine manner, enabling concept-level retrieval.

\textbf{Hierarchical Concept Tree for Structured Memory.}
Let $\{z_i = (x_i, \ell_i)\}_{i=1}^N$ denote the set of memory entries, where $\ell_i$ is the distilled lesson for problem $x_i$.
In the hierarchical concept tree, each leaf node corresponds to a subset of entries grouped by a shared concept, represented by a cluster center embedding.
The tree is constructed offline in two stages; see the upper half of Algorithm~\ref{alg:concept_tree} in Appendix.
\textbf{(1) Concept Extraction.}
For each entry $z_i=(x_i,\ell_i)$, we use the base LLM $\pi_\theta$ to extract a short, structured \emph{concept description} $d_i\defeq \pi_\theta(x_i, \ell_i)$, which abstracts the underlying topic, problem pattern, and solution technique.
For example, a geometry problem may yield:
``\texttt{Topic: angle chasing~$\|$~Problem: finding angles in cyclic quadrilaterals~$\|$~Technique: inscribed angle theorem with auxiliary lines}''. We then embed $d_i$ using an embedding model $\pi_e$, yielding $e_i \defeq \pi_e(d_i)$.
\textbf{(2) Hierarchical Clustering.} 
We construct an $L$-level hierarchical tree $\mathcal{T}$ over the embeddings $\{e_i\}_{i=1}^N$ using recursive bisecting $k$-means with cosine similarity.
Each internal node splits its data into $b$ children, where the branching factor is
$b = \lceil (\lfloor N / s \rfloor)^{1/L} \rceil$,
with $s$ the target leaf size.
A node becomes a leaf when depth $L$ is reached or its size falls below a threshold.
Each leaf $\lambda_j$ stores a subset of entries, representing a coherent concept group.
\begin{figure}[t]
\centering
\includegraphics[width=0.9\textwidth]{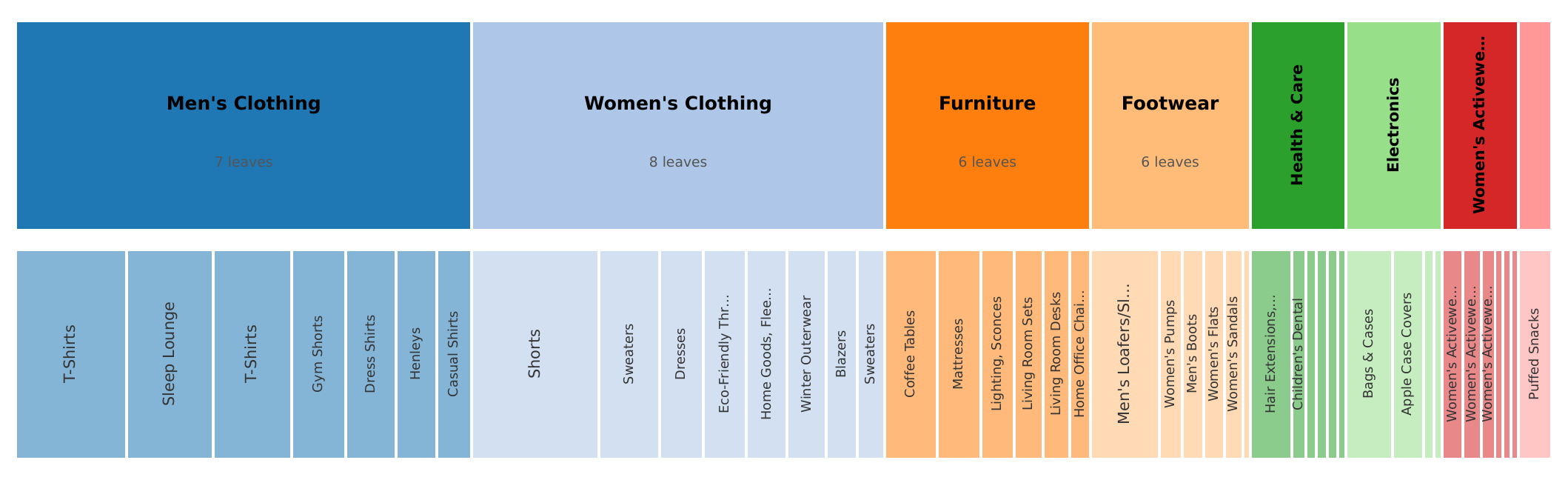}
\caption{
\textbf{Visualization of Concept Tree (WebShop).} Hierarchical concept tree constructed from 1,135 experience records on WebShop task. The tree organizes standard lesson-based memory entries into 8 broad topics and 44 concrete concept groups via hierarchical clustering. Each leaf node contains a collection of lessons, with leaf segment widths proportional to the number of records.}
\label{fig:tree_visualize}
\vspace{-1.0em}
\end{figure}

\textbf{Concept Tree Retrieval with LLM Re-ranking.}
At test time, retrieval operates at the leaf level (Algorithm~\ref{alg:concept_tree}, lower half). Given a query $x$ with embedding $e_x=\pi_e(x)$, we compute a \emph{leaf affinity score} $a_j(x)$ for each leaf $\lambda_j$ as the cosine similarity between $e_x$ and the leaf centroid. We then select the top $n_\ell$ leaves and merge their entries into a candidate pool $\operatorname{Top}\text{-}K_{\text{cand}}\{z_i\}_{i=1}^N$. This pool is much smaller than the full memory and remains \emph{structurally diversified}, since different leaves correspond to different conceptual groups. Finally, we prompt $\pi_\theta$ with the query $x$ and the candidate problem statements $\{x_i\}$, ask it to identify the most useful lessons, and format the selected entries into the final context.

\textbf{Results.} We evaluate whether this more advanced memory structure can further improve experience-augmented agents. Specifically, we construct a two-level concept tree ($L=2$) and visualize it on WebShop in Figure~\ref{fig:tree_visualize}. The first level captures broad shopping categories, while the leaf level corresponds to finer-grained subcategories, reflecting the intuition that similar products often require similar shopping strategies. We compare this agent equipped with the concept tree with a flat lesson memory $\tilde{\Mc}=\{(x_i,\ell_i)\}_{i=1}^N$ using semantic retrieval. Figure~\ref{fig:concept_tree_webshop}(a) shows that the concept tree leads to improved inference performance. Figure~\ref{fig:concept_tree_webshop}(b) shows that concept-tree-based agent tends to select lessons from a more diverse set of concept groups (the number of the concepts covered by the retrieved context) compared to the semantic-based retrieval baseline. Together with the analysis in Section~\ref{sec:analysis}, it suggests the improved performance is due to a better balance between relevance and diversity.\vspace{-1em}

\begin{figure}[h]
\centering
\begin{tabular}{rl}
    \subfloat[Performance Comparison]{%
        \includegraphics[width=0.47\textwidth]{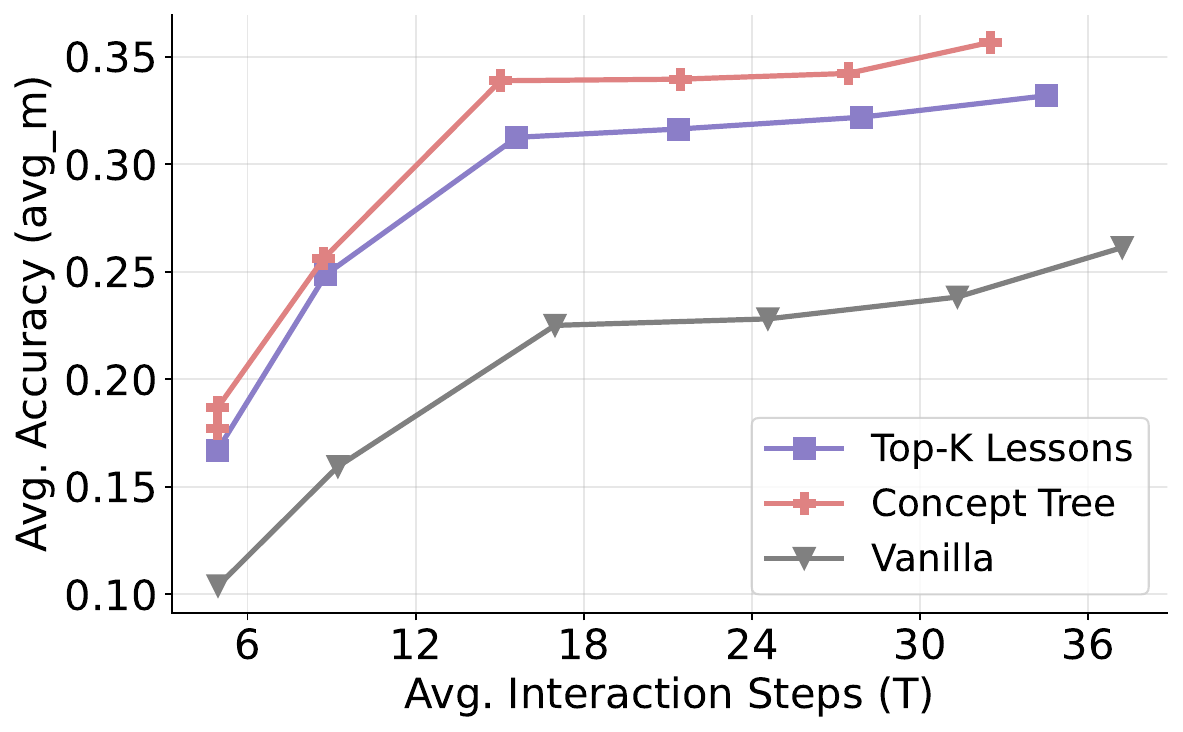}
    } &
    \subfloat[Concept Group Diversity]{%
        \includegraphics[width=0.47\textwidth]{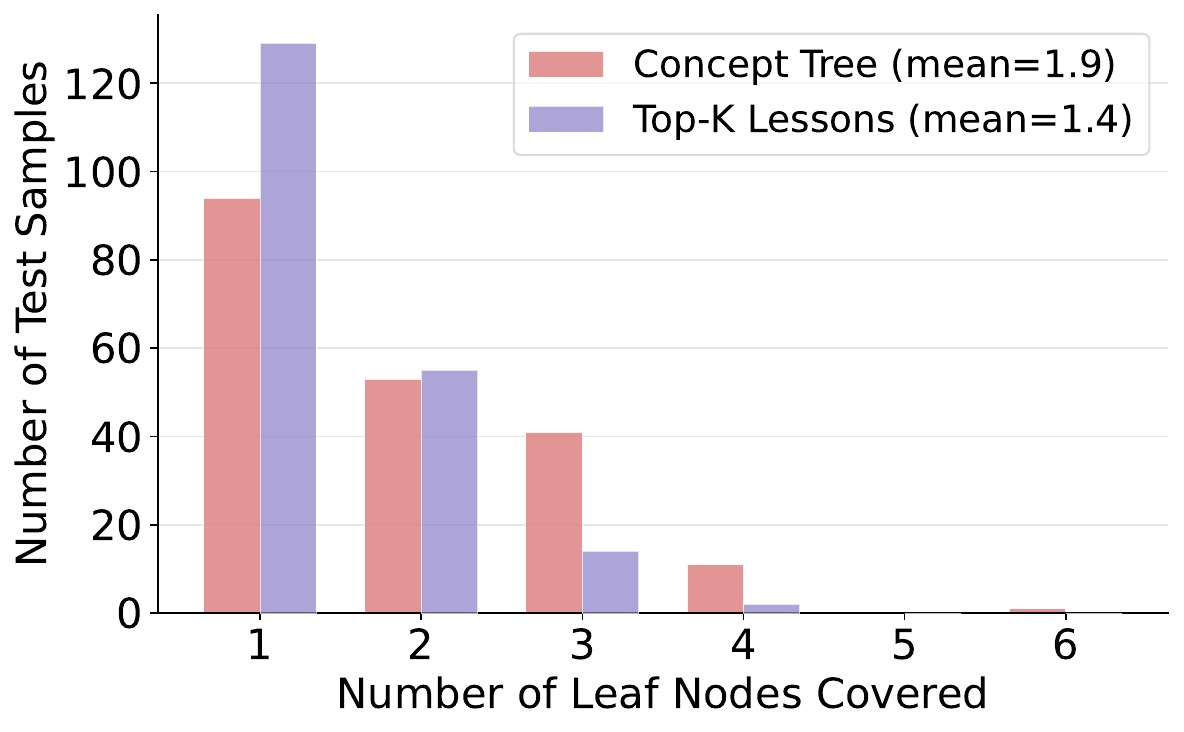}
    }
\end{tabular}
\caption{
\textbf{Hierarchical Concept Tree vs. Lesson Retrieval}
(a) Performance comparing the hierarchical concept tree against Top-$K$ lesson retrieval and vanilla baselines. (b) Concept group diversity comparison: concept tree achieves greater leaf-node coverage than Top-$K$ lesson retrieval.
}
\label{fig:concept_tree_webshop}
\vspace{-1.0em}
\end{figure}

\section{Concluding Remarks}
Across reasoning and agentic tasks, our analysis suggest that decocted experience serves as the key ingredient for effective context construction that improves the test-time inference in the LLM agent. Several directions remain for future work. One promising direction is to develop more advanced mechanisms for leveraging experience, including more sophisticated memory structures and closed-loop systems that iteratively collect new experience and update model weights to improve performance over time. Another important direction is continual learning in an agentic system: given a new task or environment, an agent should be able to accumulate new experience, integrate it with existing knowledge, and continually improve without forgetting previously acquired capabilities.

\clearpage
\bibliography{ref}
\bibliographystyle{colm2026_conference}

\clearpage
\onecolumn

\newpage
\appendix
\begingroup
\hypersetup{linkcolor=darkblue} 
\addtocontents{toc}{\protect\StartAppendixEntries}
\listofatoc
\endgroup

\section{Related Work} \label{app:related}
\paragraph{Context Engineering.}
Context engineering can be viewed as optimizing the information fed into an LLM at inference time through retrieval, generation, processing, and management, rather than changing the model parameters~\citep{mei2025survey}. Early context engineering methods typically relied on manually designed prompts for downstream tasks, including few-shot prompting~\citep{brown2020language} and prompt engineering~\citep{sahoo2024systematic}. Chain-of-thought (CoT) prompting~\citep{wei2022chain} further shows that multi-step reasoning can substantially improve performance on reasoning tasks. Beyond hand-crafted prompts, self-generated refinement~\citep{madaan2023selfrefine, shinn2023reflexion} treats the model's own reflection as an evolving context, improving generation across a range of tasks. Effective context can also be constructed from external knowledge via retrieval, such as retrieval-augmented generation (RAG)~\citep{lewis2020rag, gao2023retrieval, qian2025memorag}, which injects relevant knowledge into the prompt for downstream question answering tasks. However, in agentic settings, context is not merely evidence for answering a question, it should also encode reusable strategies and executable workflows. A growing body of recent work~\citep{zhang2025ace, wang2024agent, ouyang2025reasoningbank, wei2025evomemory} argues that agents benefit from dynamically evolving context that contains strategies and workflows, rather than from naive context summarization.

\paragraph{Agent Memory Systems.}
Transformer-based LLM agents remain constrained by a finite context window, which motivates external memory systems that organize past information while allowing the agent to retrieve relevant content efficiently at test time. Early memory architectures are primarily inference-time systems. MemGPT~\citep{packer2023memgpt} introduces an OS-inspired hierarchy that manages information across working, recall, and archival memory, while MemOS~\citep{li2025memos} builds a unified system for representation, organization, and governance across multiple memory forms. Mem0~\citep{chhikara2025mem0} focuses on scalable deployment by extracting and consolidating salient facts from multi-session interactions with graph-based memory. MemTree~\citep{rezazadeh2024isolated} further explores structured memory by organizing information into a tree representation, enabling hierarchical retrieval and integration beyond flat memory stores. MIRIX~\citep{wang2025mirix} builds a multi-agent controller that routes across specialized memory types such as episodic, semantic, and procedural memory. A more recent line of work seeks to train LLMs to \emph{learn} memory management. MemGen~\citep{zhang2025memgen} interleaves reasoning with generated latent memory, whereas Mem-$\alpha$~\citep{wang2025memalpha}, MemAgent~\citep{yu2025memagent}, and Memory-R1~\citep{yan2025memoryr1} explicitly optimize memory construction or memory operations with reinforcement learning. Despite proposing a variety of memory architectures and management strategies, most of the existing literature still centers on conversational recall~\citep{maharana2024evaluating}, user preference persistence~\citep{jiang2025know, jiang2025personamem}, or long-context question answering~\citep{yang2018hotpotqa}, where the central challenge is remembering factual information rather than reusing problem-solving strategies. In contrast, our work treats memory not merely as a storage-and-recall mechanism, but as a source of \emph{decision-support context} for future reasoning and acting. Accordingly, our focus is not only on how to store and retrieve information, but on how experience should be distilled, consolidated, and organized so that the retrieved context is most strategically useful at test time.

\paragraph{Leveraging Past Experience for Agents.}
An emerging research direction studies how agents self-improve from prior interaction experience, either at test time or through weight updates during training. An early representative test-time approach is ExpeL~\citep{zhao2024expel}, which collects trajectories on training tasks, abstracts them into insights, and retrieves both insights and successful past experiences at test time, showing that agents can improve from accumulated experience without gradient updates. More recent test-time approaches further develop this idea in various forms. Dynamic Cheatsheet~\citep{suzgun2025dynamic} equips LLMs with an adaptive memory of reusable strategies and problem-solving insights extracted from past attempts. ReasoningBank~\citep{ouyang2025reasoningbank} distills lessons from both successful and failed trajectories into reasoning memory and couples the agent with memory-aware test-time scaling. Under a similar setting, Evo-Memory~\citep{wei2025evomemory} provides a streaming benchmark that evaluates whether agents can retrieve, adapt, and evolve memory using accumulated interactions. In parallel, training-time approaches such as Early Experience~\citep{zhang2025earlyexperience} propose reward-free supervision derived from an agent's own experience, where observed future states serve as supervision. Closely related, DreamGym~\citep{chen2025scaling} proposes scaling agent learning via \emph{experience synthesis}: instead of relying on expensive real-environment rollouts, it trains agents with a reasoning-based experience model that generates synthetic transitions, feedback, and progressively more challenging tasks for online RL. ExRL~\citep{shi2026erl} embeds an explicit experience--reflection--consolidation loop inside RL training: a first attempt receives feedback, self-reflection guides a second attempt, and successful improvements are distilled into the base policy to encourage self-improvement. Our work is complementary to this line of research, but focuses on a different question: rather than proposing yet another framework for leveraging experience, we ask what makes past experience actually useful for future inference. Our key insight is that experience becomes effective not simply by being stored, but by being \emph{decocted} into context: extracting the essence from experience, organizing it coherently, and retrieving salient information to to balance relevance with diversity, thereby revealing the principles that make experience-augmented agent effective and scalable.

\section{Proof of Proposition~\ref{prop:efficiency}}
\label{app:proof}

\begin{proof}
Consider
\begin{align}
H(Y \mid X=x,C=c)
&\stackrel{(a)}{\ge} \E_{p(\tau|x,c)}\Bigl[H(Y \mid \tau,X=x,C=c)\Bigr]\nonumber\\
&\stackrel{(b)}{\ge} \E_{p(\tau|x,c)}\Bigl[H(Y_{1:\tau} \mid \tau,X=x,C=c)\Bigr]\nonumber\\
&\stackrel{(c)}{\ge} h\E_{p(\tau|x,c)}[\tau].
\label{eq:intermed}
\end{align}
Here, $(a)$ and $(b)$ follow since conditioning and dropping variables reduce entropy, respectively; $(c)$ follows from  the lower-bound assumption on the conditional entropy.

Dividing both sides by $h>0$, we have
\begin{align*}
\E_{p(\tau|x,c)}[\tau]=\E[\tau \mid X=x, C=c]
&\le
\frac{H(Y \mid X=x,C=c)}{h}\\
&=
\frac{H(Y \mid X=x) - I(Y; C=c \mid X=x)}{h},
\end{align*}
where the last equality follows from the definition of information gain.
\end{proof}
\section{Experimental Setup} \label{app:setting}

\subsection{Datasets} \label{app:datasets}
\paragraph{Mathematical Reasoning.}
We use the DAPO-Math dataset~\citep{yu2025dapo} to construct the experience memory. The training set contains 14,116 competition-level mathematical reasoning problems spanning a diverse set of topics, including algebra, geometry, number theory, combinatorics, and probability. For evaluation, we consider six challenging math reasoning benchmarks: AMC 2023, AIME 2024, AIME 2025, HMMT 2024, HMMT 2025, and BeyondAIME~\citep{bytedance_seed_2025_beyondaime}, for a total of 260 test problems. These benchmarks are drawn from highly competitive high-school and university-level mathematics contests. The problems are intentionally difficult and require substantially deeper reasoning than standard school-level math benchmarks. Since both the training set and the evaluation benchmarks provide ground-truth final answers, we can automatically verify model outputs and assign binary feedback \(r \in \{0,1\}\) based on whether the generated final answer is correct.

\paragraph{WebShop Agent.}
WebShop~\citep{yao2022webshop} is a large-scale interactive web environment that simulates online shopping. In this environment, an agent is given a natural-language instruction describing a target product and must navigate an e-commerce website to identify, configure, and purchase an item that satisfies the user’s request. The environment is built from more than 1.18 million real-world Amazon products and 12,087 crowd-sourced shopping instructions, and each episode unfolds over multiple webpage types, including search, results, item, and item-detail pages. At each step, the agent observes the current webpage and can take high-level actions such as issuing a search query, selecting a product from the results list, choosing product options, opening detailed descriptions or overviews, navigating between pages, and finally clicking \texttt{Buy} to terminate the episode. Following the WebShop setup in AgentGym~\citep{xi2024agentgym}, we use 3,930 training trajectories to construct memory and 200 test episodes for evaluation. Performance is measured by a scalar reward \(r \in [0,1]\) assigned to the final purchased item, which reflects how well it matches the user instruction in terms of product type, attributes, options, and price.

\paragraph{Software Engineering Agent.}
We consider software engineering tasks based on SWE-bench~\citep{jimenez2023swe}, a benchmark constructed from real GitHub issues and their corresponding pull requests across 12 popular Python repositories. Each instance provides an issue description and a pre-fix snapshot of the repository, and the agent must generate code edits that resolve the issue in the existing codebase. We adopt the SWE-agent framework~\citep{yang2024swe}, in which the agent operates in a command-line environment with access to standard bash commands as well as specialized tools for repository navigation, file viewing, content search, and code editing. This setting requires the agent to localize the relevant files, understand the repository structure, implement the fix, and interact with the execution environment to check whether the modification is successful. Following the SWE-bench evaluation protocol, an instance is counted as resolved (i.e., \(r \in \{0,1\}\)) only if the generated patch applies successfully and all associated fail-to-pass and pass-to-pass tests succeed, ensuring both issue resolution and preservation of existing functionality. For memory construction, we use the SWE-bench full set excluding SWE-bench Verified, yielding 1,794 training instances, and we evaluate on SWE-bench Verified, a human-validated test set of 500 instances designed for more reliable assessment.

\subsection{Implementation Details} \label{app:exp_details}
\paragraph{Raw Memory Construction.}
The memory stores the agent's past experience from interactions on the training set. For each training sample, we let the LLM generate $m=4$ random trajectories using a sampling temperature of $1.0$ and an unlimited thinking budget. We then retain only those samples for which the feedback signals across the four attempts are not all zero, ensuring that each memory entry contains at least one positive experience. The raw memory is stored as a dictionary containing the problem statement, the raw experience, and the corresponding feedback signals. 

\paragraph{Lesson Distillation.}
The lesson distillation procedure introduced in Section~\ref{subsec:context} uses the agent itself to extract a reusable lesson from raw experience, i.e., a problem-solving strategy distilled from the $m=4$ sampled attempts. For successful attempts, the agent is prompted to summarize key insights and common reasoning patterns. For failed attempts, it is prompted to reflect on common mistakes and reasoning flaws that should be avoided in future problems. We use greedy decoding for lesson extraction. 

\paragraph{Memory Consolidation.}
The memory consolidation procedure introduced in Section~\ref{subsec:memory} applies standard $k$-means clustering in the embedding space. The hyperparameter for combining the problem-statement embedding and the lesson embedding is set to $\alpha = 0.5$. Consolidation is applied after lesson distillation, starting from the lesson-based memory $\tilde{\Mc}=\{(x_i,\ell_i)\}_{i=1}^N$, with the goal of compressing it into a smaller memory with $\tilde{N}$ entries. Specifically, the $N$ memory entries are partitioned into $\tilde{N}$ clusters, and only the entry closest to each cluster centroid is retained, while the remaining cluster members are removed.

\paragraph{Context Retrieval.}
Our default retrieval method is based on semantic similarity. Specifically, we use Qwen3-Embedding-4B~\citep{qwen3embedding} to encode the problem statement of each memory entry into a 2,560-dimensional embedding vector. Given a query, the retriever selects the top-$K$ memory entries with the highest embedding similarity to the query problem. By default, we set $K=8$ for math reasoning, $K=4$ for WebShop, and $K=4$ for SWE. The final context is constructed from the query problem statement, the problem statements of the retrieved memory entries, and their corresponding experiences or lessons.

\paragraph{Inference and Evaluation.}
The constructed context is fed into the LLM for inference. During evaluation, we set the sampling temperature to $1.0$ and generate $m=8$ random outputs. We evaluate the agent under a fixed test-time budget. For standard reasoning tasks such as math reasoning, we impose a maximum output-token budget, and a problem is counted as solved only if the model produces the correct answer within this token limit. For agentic tasks, we instead define the test-time budget as the maximum number of interaction steps allowed between the agent and the environment. We do not impose an additional token budget within each interaction step, since in agentic settings the primary cost typically arises from repeated rounds of action and feedback, rather than from a long reasoning trace within a single step. The metrics introduced in Section~\ref{subsec:setup} are computed by averaging over the $m=8$ sampled outputs.

\paragraph{Concept Tree.}
The concept tree introduced in Section~\ref{sec:concept_tree} is built on top of the lesson-based memory $\{z_i = (x_i, \ell_i)\}_{i=1}^N$. For each memory entry, we first perform \emph{concept description extraction} by prompting the base LLM to produce a short structured description. We use greedy decoding in this stage. These concept descriptions are then encoded into embeddings and used to construct the tree via recursive bisecting $k$-means clustering. For all tasks, we use a two-level hierarchy ($L=2$). The target leaf size is set to $50$ for mathematical reasoning, $20$ for WebShop, and $10$ for SWE, reflecting the different memory sizes in each domain. 

At test time, retrieval is performed at the leaf level. Given a query, we first rank leaves by their affinity to the query embedding and select the top $n_\ell=5$ leaves. We then pool the memory entries contained in these leaves to form a candidate set for LLM re-ranking. In practice, we set a maximum candidate pool size $K_{\mathrm{cand}} = 300$ so that it is typically within $300$ entries, which keeps the re-ranking stage within the LLM context window while still preserving sufficient diversity. Finally, the LLM is prompted with the query and the candidate problem statements, and it is allowed to select a variable number of relevant lessons up to a predefined cap. This design gives the model flexibility to adapt the amount of retrieved context to the different query problems, rather than forcing a fixed number of retrieved lessons for every instance.

\subsection{Prompt Templates} \label{app:prompts}
\begin{tcolorbox}[gray_box, title = {{Prompt Template: Lesson Distillation (Math Reasoning)}}]\footnotesize \label{prompt:lesson_math}

You are a problem-solving strategy synthesizer analyzing problem-solving attempts. Carefully review the problem, and your attempted solutions (can be correct or incorrect). Your goal is to extract a detailed problem-solving strategy from a long reasoning process that can help solve this type of problems.\\

\# Problem Statement: \\
\{problem_statement\} \\

\# Ground Truth Solution: \\
\{gt_solution\} \\

\# Previous Attempts:\\
\{formatted_rollouts\}\\

\# Reasoning Instructions:\\
Extract key insights and reasoning patterns from successful attempts, aggregating different ways to solve the problem.

Reflect Common mistakes, misunderstandings, or reasoning flaws that could lead to incorrect answers (if applicable).

Analyze preconditions and scenarios where the identified strategies are most applicable.

\# Response Format:\\
1. Task Description: A one-two sentence summary of the type of problems this strategy applies to.\\
2. Strategy: A step-by-step detailed problem-solving strategy that could consist of various different ways to tackle similar problems.\\
3. Pitfalls: Common mistakes or misconceptions to avoid when solving this type of problems (if applicable).\\

\# Important: \\
The strategy should be extremely detailed, covering multiple different ways to solve the problem.
\end{tcolorbox}
\begin{tcolorbox}[gray_box, title = {{Prompt Template: Lesson Distillation (WebShop)}}]\footnotesize
\label{prompt:lesson_webshop}
You are a shopping strategy synthesizer analyzing past web shopping attempts. 
Review the shopping task and the previous interaction trajectories (can be successful, partially successful, or unsuccessful based on the reward). Your goal is to extract a detailed reusable shopping strategy from a long reasoning process that can help solve this type of problems.\\

\# Shopping Task:\\
\{task_instruction\}\\

\# Past Shopping Trajectories:\\
\{formatted_trajectories\}\\

\# Instructions:\\
Extract key insights and reasoning patterns from successful attempts (reward = 1.0).

Reflect Common mistakes, misunderstandings, or reasoning flaws that could lead to sub-optimal attempts (reward $<$ 1.0)

Analyze preconditions and scenarios where the identified strategies are most applicable \\

\# Response Format:\\
1. Task Description: A one-two sentence summary of the type of shopping task this strategy applies to.\\
2. Action Workflow: A step-by-step detailed problem-solving strategy to complete the purchase, including step-level actions and the summarization of the observations at each step.\\
3. Pitfalls: Common mistakes or pitfalls to avoid when solving this type of shopping task (if applicable).\\

\# Important:\\
The workflow should include clear action steps (e.g., CLICK, TYPE, SELECT) and observations at each step.
\end{tcolorbox}
\begin{tcolorbox}[gray_box, title = {{Prompt Template: Lesson Distillation (SWE)}}]\footnotesize
\label{prompt:lesson_swe}
You are a debugging strategy synthesizer analyzing software engineering problem-solving attempts. Review the GitHub issue description and the previous resolution trajectories below (can be successful or unsuccessful based on the result). Your goal is to distill a detailed reusable debugging and fixing strategy from a long reasoning process that can help solve this type of problems. \\

\# GitHub Issue:\\
\{issue_text\}\\

\# Past Resolution Trajectories:\\
\{formatted_trajectories\}\\

\# Reasoning Instructions:\\
Extract key insights and reasoning patterns from successful attempts (resolved = True).\\
Reflect on common mistakes or reasoning flaws from failed attempts (resolved = False).\\
Focus on actionable details: concrete file paths, class/function names, grep patterns, test commands, and verification steps.\\

\# Response Format:\\
1. Task Description: One-two sentence summary of the issue type and code base area.\\
2. Strategy: Step-by-step detailed agent workflow structured as an action sequence \\(followed by successful attempts).\\
3. Pitfalls: Common mistakes to avoid. If a failed attempt is available, explain the specific wrong turn and how to avoid it.\\

\# Important: \\
The workflow should include clear action steps and observations at each step.
\end{tcolorbox}
\begin{tcolorbox}[gray_box, title = {{Prompt Template: Experience-based Inference}}]\footnotesize
\label{prompt:inference}
You are an expert at solving complex reasoning problems, while leveraging few-shot experience. You have been given several examples from past experience that may help solve the new target problem. Please solve the problem in an efficient manner by leveraging the past experience. \\
 
\# Reasoning Instructions:\\
\{task_dependent_instruction\} \\

\# Few-shot Experience:\\
\{examples\} \\

\# Target Problem Statement: \\
\{problem_statement\}
\end{tcolorbox}
\begin{tcolorbox}[gray_box, title = {{Prompt Template: Concept Description Extraction (Math Reasoning)}}]\footnotesize
\label{prompt:concept_math}
You are a mathematical concept analyst.  Given a problem statement and a lesson/solution, produce a concise concept description that captures the underlying mathematical topic, the core problem pattern, and the key technique(s) used. \\

Problem: \\
\{problem_statement\} \\

Lesson / Solution: \\
\{lesson\} \\

Respond with EXACTLY three lines (no extra text):\\
Topic: [1-4 word topic, e.g. ``number theory'', ``dynamic programming'']\\
Problem: [1 sentence summarizing the problem pattern]\\
Technique: [1 sentence summarizing the solution technique]
\end{tcolorbox}
\begin{tcolorbox}[gray_box, title = {{Prompt Template: Concept Description Extraction (WebShop)}}]\footnotesize
\label{prompt:concept_webshop}
You are a shopping task analyst.  Given a task instruction and a lesson from an agent's shopping experience, produce a concise concept description that captures the shopping category, the core task pattern, and the key strategy used.\\

Task Instruction:\\
\{problem_statement\} \\

Lesson:\\
\{lesson\} \\

Respond with EXACTLY three lines (no extra text):\\
Topic: [1-4 word shopping/product category, e.g. ``clothing, woman'', ``electronics, laptop'']\\
Problem: [1 sentence summarizing the shopping task]\\
Technique: [1 sentence summarizing the shopping strategy and agentic workflow]
\end{tcolorbox}
\begin{tcolorbox}[gray_box, title = {{Prompt Template: Concept Description Extraction (SWE)}}]\footnotesize
\label{prompt:concept_swe}
You are a software engineering issue analyst.  Given a GitHub issue description and a debugging lesson/solution, produce a concise concept description that captures the repository/library area, the core bug pattern, and the key fix technique used.\\

Issue: \\
\{problem_statement\} \\

Lesson / Solution: \\
\{lesson\}\\

Respond with EXACTLY three lines (no extra text):\\
Topic: [1-4 word topic, e.g. ``Django ORM query'', ``NumPy indexing'', ``pytest fixture''] \\
Problem: [1 sentence summarizing the bug pattern]\\
Technique: [1 sentence summarizing the fix approach]
\end{tcolorbox}
\begin{tcolorbox}[gray_box, title = {{Prompt Template: LLM Re-ranking}}]\footnotesize
\label{prompt:reranking}
You are an expert relevance judge for problem-solving experience retrieval.\\

Given a **query problem** and a numbered list of **candidate experience entries** (each described by its problem statement), select the most relevant candidates based on problem similarity (the same type of problem that shares similar solution strategies).\\

\#\# Query Problem:\\
\{query_problem_statement\}\\

\#\# Candidate Experiences:\\
\{formatted_candidates\}\\

\#\# Selection Guidelines:\\
- You may select as many candidates as you think are relevant, but no more than {max_k}.\\
- Order your selections from most relevant to least relevant.\\
- If for some rare case you find that none of any candidate is relevant, select {top_k} candidates that could be helpful to the query problem.\\

Return ONLY a JSON array of integer indices (0-based), Example: [4, 12, 7, 0, 3, 18, 24, 1, 9, 53]\\
Your response MUST end with the JSON integer array on its own line 
with no surrounding text or labels. Output (JSON integer array only, on the last line):
\end{tcolorbox}
\section{Additional Experiment Results}  \label{app:results}

\subsection{Omitted Results in Section~\ref{sec:concept_tree}}
\label{app:concept_tree_omitted}

\paragraph{Algorithm of Concept Tree Construction and Retrieval.}
\begin{algorithm}[h]
\caption{Hierarchical Concept Tree}
\label{alg:concept_tree}
\tcp{\textbf{Offline: Build Concept Tree}}
\KwIn{
Query-lesson pairs $
\{(x_i,\ell_i)\}_{i=1}^N$ 
tree depth $L$;\; target leaf size $s$;\; number of selected leaves $n_\ell$;\; candidate pool size $K_{\mathrm{cand}}$.
}
\KwOut{Hierarchical concept tree $\Tc$.}

\For{$i=1$ \KwTo $N$}{
    $d_i \leftarrow \pi_\theta(x_i, \ell_i)$ \tcp*{extract concept description}
    $e_i \leftarrow \pi_e(d_i)$ \tcp*{embed concept description}
}
$\mathcal{T} \leftarrow \textsc{BisectingKMeans}\!\big(\{e_i\}_{i=1}^N,\; L,\; s\big)$ \tcp*{build $L$-level concept tree}

\BlankLine
\hrule
\BlankLine
\tcp{\textbf{Online: Retrieve \& Re-rank}}\vspace{.25em}
\KwIn{
Tree $\Tc$; test query $x$.
}
\KwOut{Context $c$ for query $x$.}

$e_x \leftarrow \pi_e(x)$ \tcp*{embed query}

\For{each leaf $\lambda_j$ in $\mathcal{T}$}{
    $a_j(x) \leftarrow \text{Affinity}(e_x,\; \lambda_j)$ \tcp*{compute cosine similarity to leaf nodes}
}

$\Lambda^* \leftarrow \operatorname{Top\text{-}n_\ell}_{\lambda_j}\; a_j(x)$ \tcp*{select top-$n_\ell$ leaves}

$\mathcal{C}_{\mathrm{cand}} \leftarrow \operatorname{Top\text{-}K_{\mathrm{cand}}}\bigl\{z_i : i \in \bigcup_{\lambda_j \in \Lambda^*}\mathrm{members}(\lambda_j)\bigr\}$ \tcp*{pool members of selected leaves}

$\mathcal{R}(x) \leftarrow \pi_\theta\!\big(x,\;\{x_i : z_i \in \mathcal{C}_{\mathrm{cand}}\}\big)$ \tcp*{LLM re-ranking}

$c \leftarrow \mathrm{Concat}\big(x,\;\{(x_i, \ell_i)\}_{i\suchthat z_i \in \mathcal{R}(x)}\big)$\;

\Return $c$\;
\end{algorithm}
We first provide the detailed algorithm of the hierarchical concept tree in Algorithm~\ref{alg:concept_tree}. At a high level, the method consists of an offline construction stage and an online retrieval stage. Offline, each lesson memory entry $(x_i,\ell_i)$ is converted into a short concept description by the base LLM, and these concept descriptions are embedded and organized into a multi-level tree via recursive bisecting $k$-means clustering. Online, given a test query, we first score leaf nodes by their affinity to the query embedding, then select the top-ranked leaves to form a candidate pool, and finally use the LLM to re-rank candidate entries before constructing the final context. This procedure allows retrieval to operate at the level of concept groups rather than individual memory entries, which provides a more structured way to control retrieval relevance and promote diversity beyond flat similarity-based retrieval.

\paragraph{Concept Tree Visualization.}
Figures~\ref{fig:tree_visualize_math_swe}(a) and~\ref{fig:tree_visualize_math_swe}(b) visualize the resulting concept trees on the math reasoning and SWE tasks, respectively. On math reasoning, the root-level nodes correspond to broad mathematical topics such as combinatorics, number theory, geometry, and algebra, while lower-level leaves further specialize into finer-grained problem types and strategies, such as modular arithmetic, linear congruences, cyclic quadrilaterals, or triangle area ratios. On SWE, the tree exhibits a different but equally interpretable structure: high-level nodes often correspond to software libraries, subsystems, or bug families, while leaf nodes capture more concrete repair patterns, such as Django middleware issues, SymPy set operations, Matplotlib configuration bugs, or PyTest setup failures. 
\begin{figure}[h]
\centering
\begin{tabular}{rl}
    \subfloat[Math Reasoning]{%
       \includegraphics[width=1.0\textwidth]{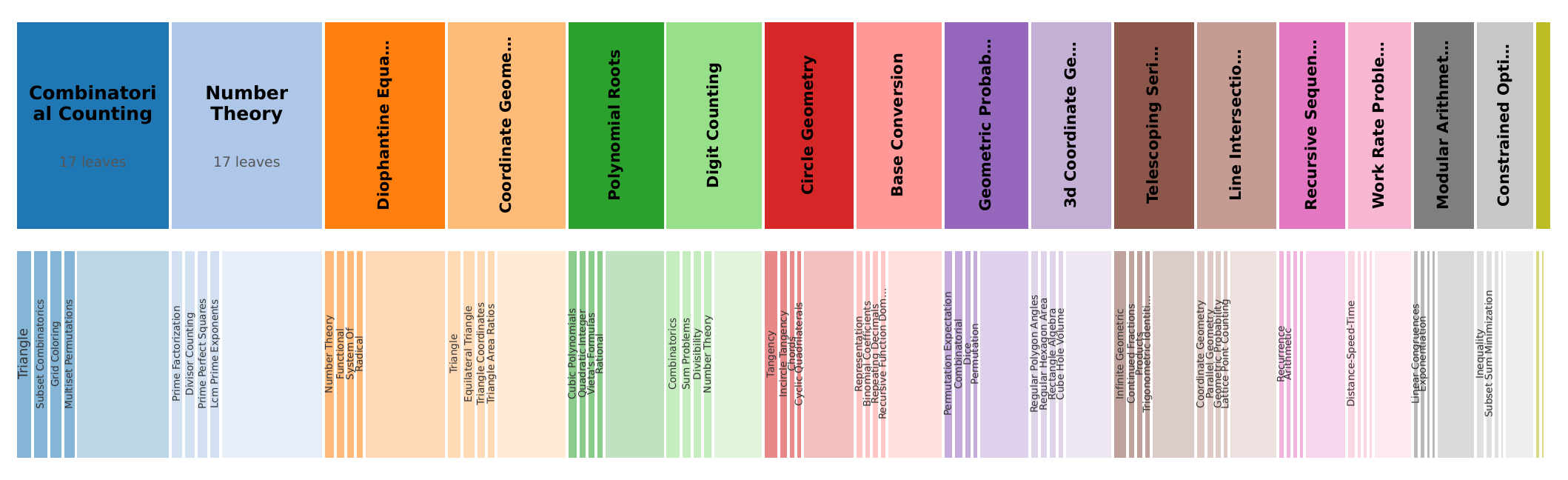}
    } \\
    \subfloat[SWE]{%
       \includegraphics[width=1.0\textwidth]{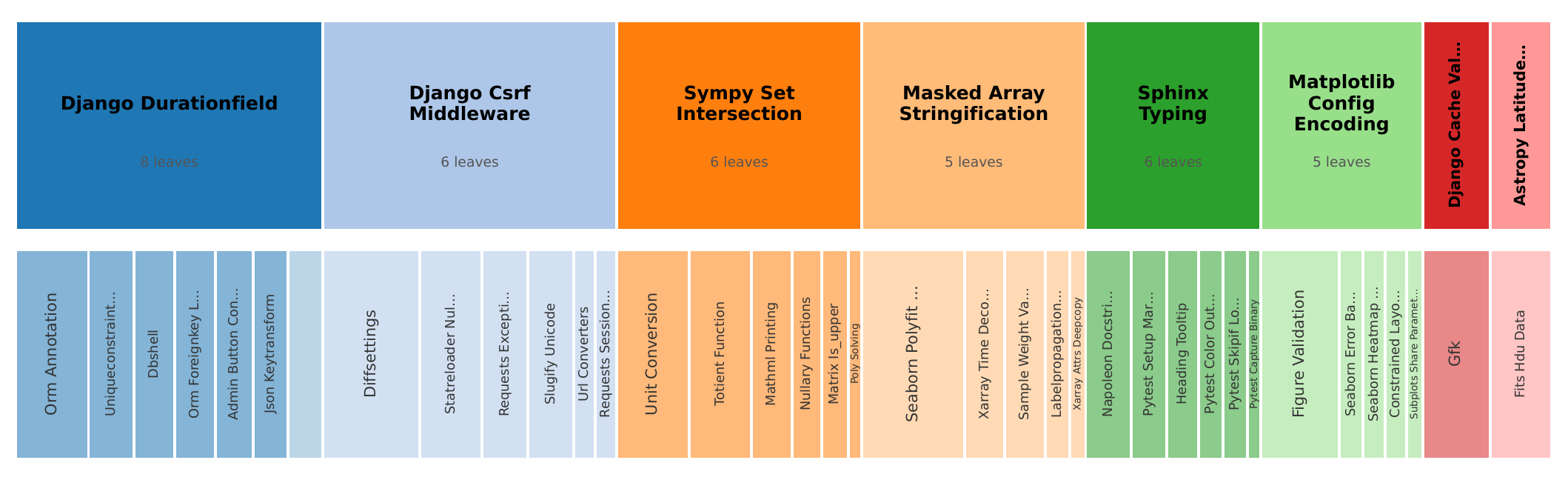}
    }
\end{tabular}
\caption{
\textbf{Visualization of Concept Tree (Math Reasoning \& SWE).} (a) Hierarchical concept tree constructed from 13,381 experience records on the math reasoning task. The tree organizes memories into 17 broad topics and 256 concrete concept groups via two levels of embedding-based clustering, with segment widths proportional to the number of records in each cluster. For visual clarity, only the four largest leaf clusters per root category are shown individually, and the remaining leaves are merged. (b) Hierarchical concept tree constructed from 579 experience records on the SWE task. The tree organizes memories into 8 broad topics and 38 concrete concept groups.}
\label{fig:tree_visualize_math_swe}
\vspace{-0.5em}
\end{figure}

\begin{figure}[h]
\centering
\begin{tabular}{rl}
    \subfloat[Math Reasoning]{%
        \includegraphics[width=0.48\textwidth]{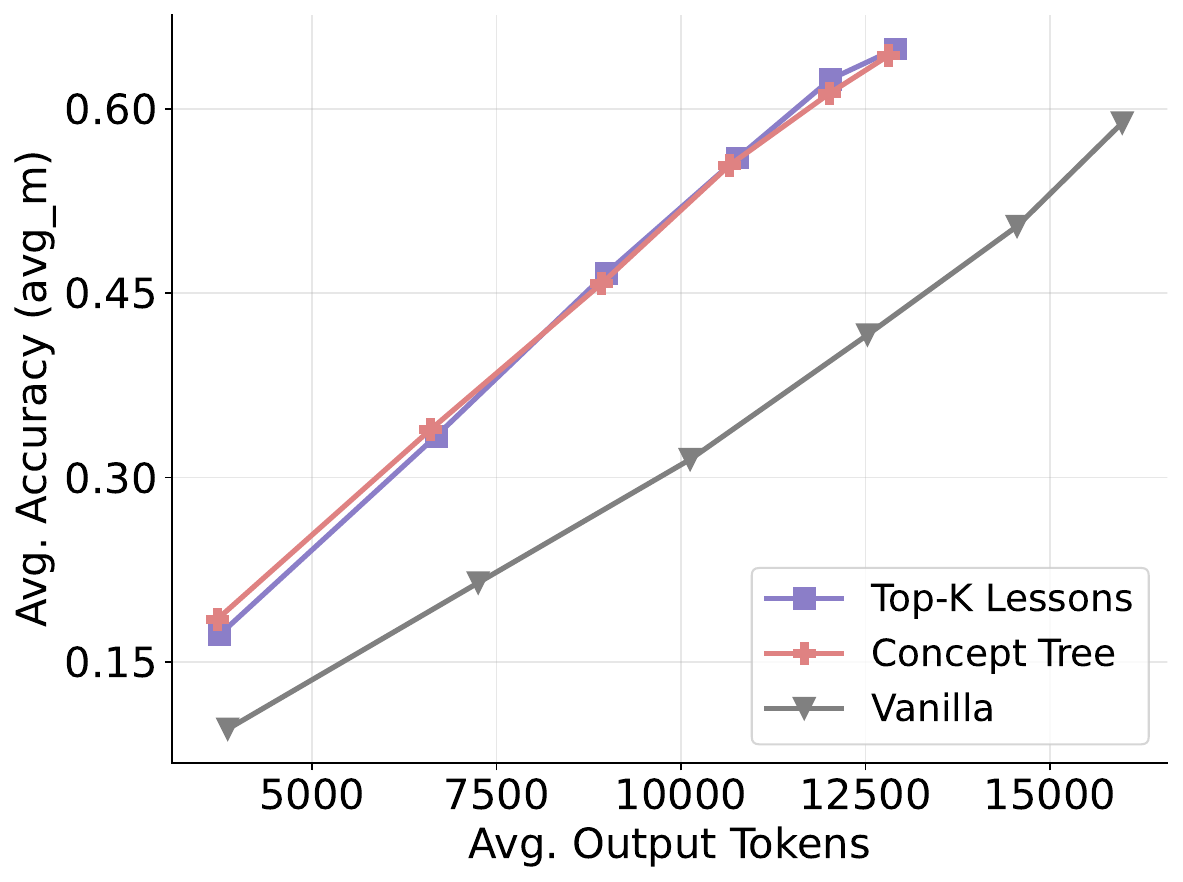}
    } &
    \subfloat[SWE]{%
        \includegraphics[width=0.48\textwidth]{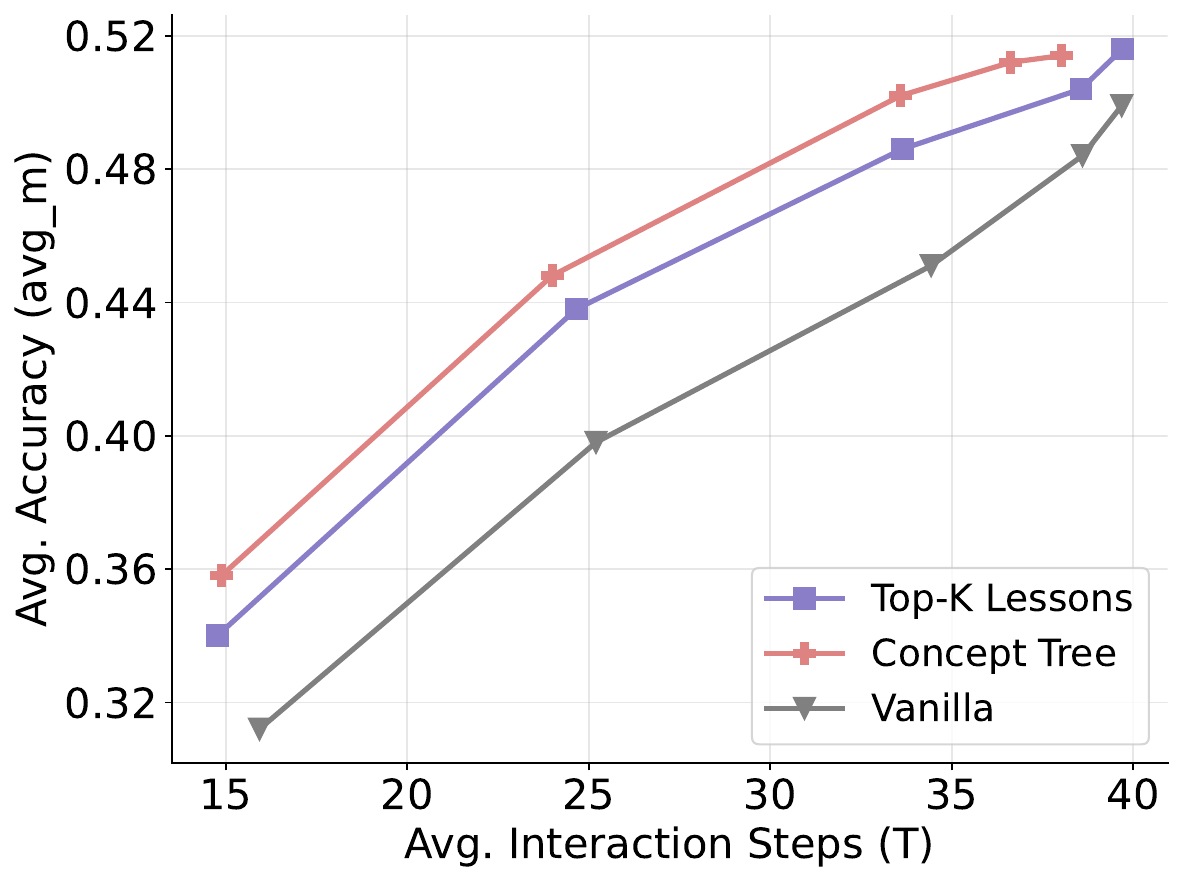}
    }
\end{tabular}
\caption{
\textbf{Hierarchical Concept Tree vs. Lesson Retrieval} Performance comparing the hierarchical concept tree against Top-$K$ lesson retrieval and vanilla baselines.}
\label{fig:concept_tree_math_swe}
\vspace{-0.5em}
\end{figure}

\paragraph{Results on Math Reasoning and SWE Tasks.}
Figure~\ref{fig:concept_tree_math_swe} further compares the hierarchical concept tree with the top-$K$ lesson retrieval baseline on the math reasoning and SWE tasks. On math reasoning, the concept tree does not yield a clear improvement over flat lesson retrieval. One possible explanation is that math problems and their distilled lessons are already relatively clean, diverse, and semantically focused, so standard embedding-based retrieval over a flat lesson memory is often sufficient to identify useful examples. In contrast, on SWE, the concept tree still delivers a noticeable improvement over the baseline, consistent with our observation on the WebShop task (Figure~\ref{fig:concept_tree_webshop}). We hypothesize that this is because experience from agentic tasks often shares deeper conceptual structure, such as reusable workflows, debugging patterns, or web navigation strategies, that may not be fully captured by surface-level semantic similarity alone.

\subsection{Analysis on Other Base LLMs.}
We further conduct the analysis in Section~\ref{sec:scaling_behavior} with GPT-OSS-20B as the base model, using WebShop as a representative task for illustration. The overall observations remain consistent with the main results: lesson distillation continues to outperform raw experience on the agentic task, and memory consolidation still exhibits a non-monotonic scaling behavior with a sweet spot at intermediate memory sizes. These results suggest that the main conclusions of Section~\ref{sec:scaling_behavior} are not specific to Seed-OSS-36B-Instruct, but instead reflect a more general property of experience-augmented agent. In particular, they further support our central claim that effective \emph{experience decoction}, through distilling noisy raw trajectories into reusable lessons and organizing accumulated memory appropriately, is key to improving the scalability and effectiveness of test-time context construction.

\begin{figure}[h]
\centering
\begin{tabular}{rl}
    \subfloat[Raw Experience vs. Distilled Lesson]{%
        \includegraphics[width=0.48\textwidth]{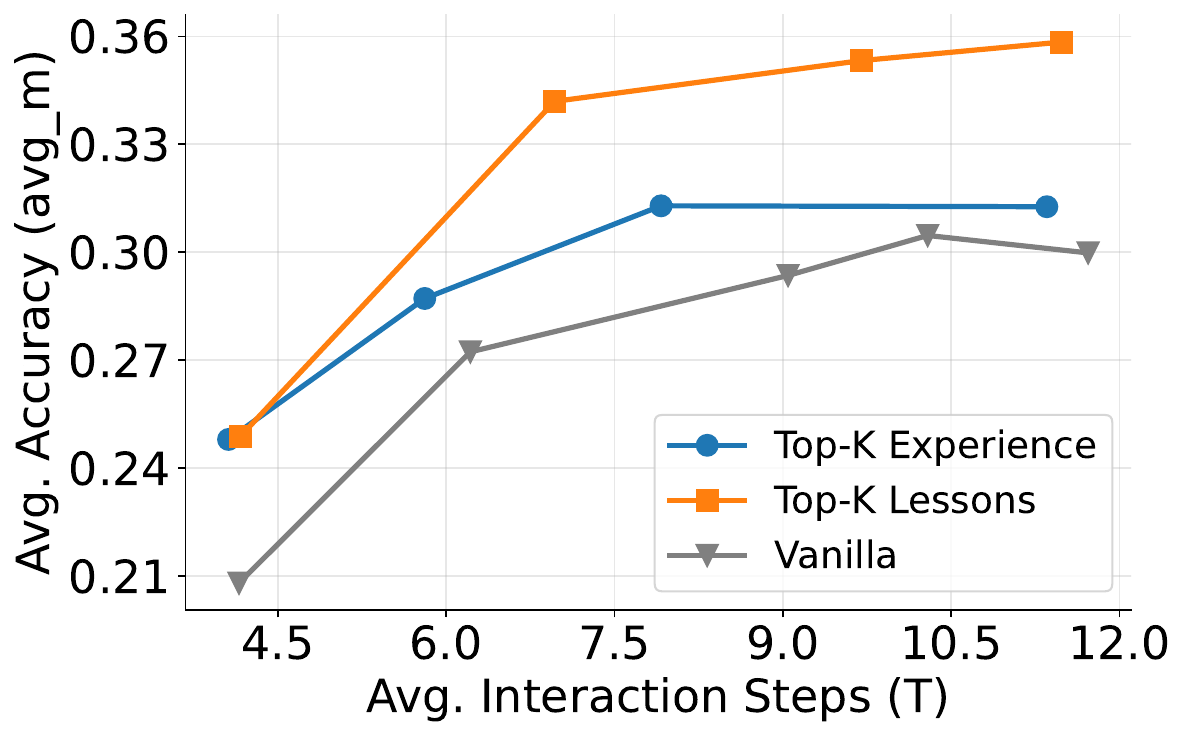}
    } &
    \subfloat[Experience Scaling Behavior via Memory Consolidation]{%
        \includegraphics[width=0.48\textwidth]{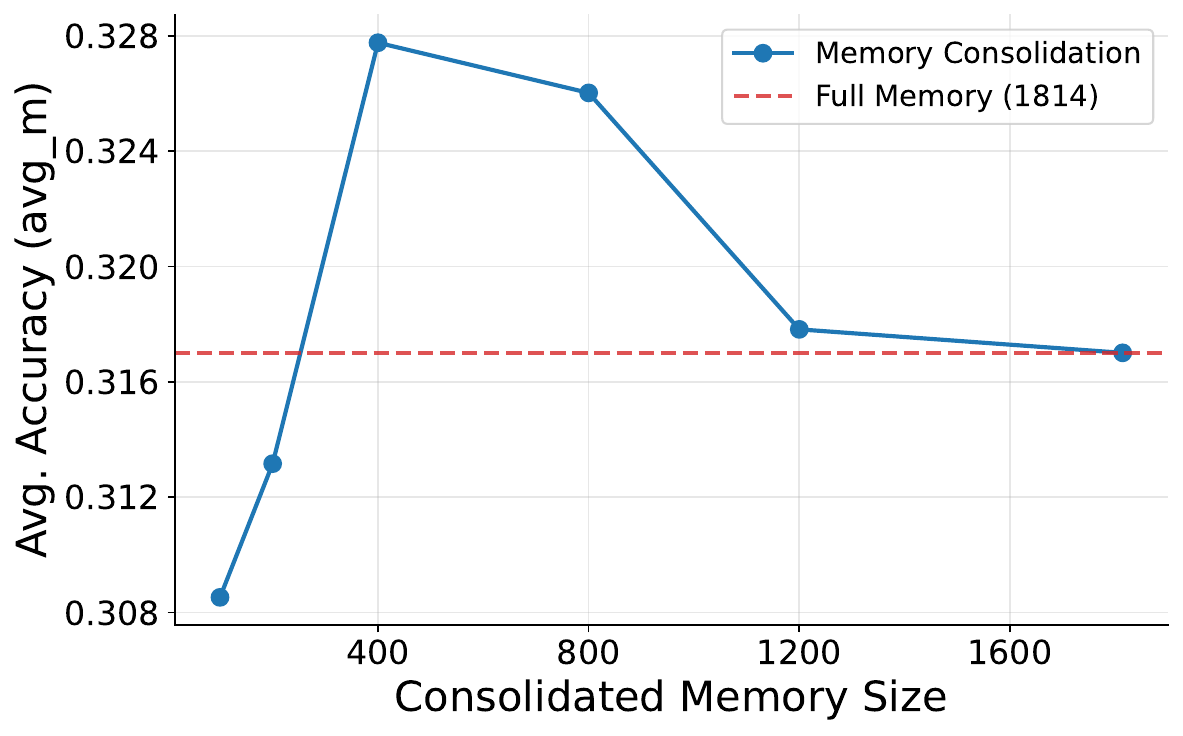}
    }
\end{tabular}
\caption{\textbf{Analysis Results on GPT-OSS-20B (WebShop).} The overall trends are consistent with the main results in Section~\ref{sec:scaling_behavior}. \textbf{(a)} Distilled lessons provide more effective context than raw experience. \textbf{(b)} Memory consolidation shows a sweet spot at intermediate memory sizes.}
\label{fig:results_gpt}
\end{figure}

\end{document}